\documentclass[11pt]{article}

\usepackage[utf8]{inputenc} % allow utf-8 input
\usepackage[T1]{fontenc}    % use 8-bit T1 fonts
\usepackage{hyperref}       % hyperlinks
\usepackage{url,authblk}            % simple URL typesetting
\usepackage{booktabs}       % professional-quality tables
\usepackage{amsfonts}       % blackboard math symbols
\usepackage{nicefrac}       % compact symbols for 1/2, etc.
\usepackage{microtype}      % microtypography
\usepackage{xcolor}         % colors

\usepackage[round]{natbib}
\usepackage{graphicx}
\usepackage{amsmath,amssymb,mathtools,amsthm}
\usepackage{mathrsfs}
\usepackage{dirtytalk}
\usepackage{tikz}
\usepackage{placeins}
\usepackage{csvsimple}
\usepackage{siunitx}
\usepackage{algorithm}
\usepackage{algorithmic}
\newcommand{\KLf}[2]{\operatorname{KL}(#1 \mid\mid #2)}
\newcommand{\E}{\mathbb{E}}
\newcommand{\R}{\mathbb{R}}

\newcommand{\SHK}{\mathrm{SHK}}
\newcommand{\Loc}{\operatorname{Loc}}
\newcommand{\amin}{a_{\min}}
\newcommand{\Dsep}{\Delta_{\min}}
\newcommand{\aspace}{\Delta_{M,\amin}^{\circ}}
\newcommand{\locspace}{\Loc_{M,\Dsep}}
\newcommand{\softmin}{\operatorname{softmin}}
\newcommand{\ProbM}{\mathcal{P}_{M,\amin,\Dsep}(\Omega)}
\newcommand{\rloc}{r^{\operatorname{loc}}}

\newcommand{\etaretarg}[1]{\eta_{\operatorname{ret},#1}}

\newcounter{lemmano}
\newcounter{theoremno}

\newtheorem{theorem}[theoremno]{Theorem}
\newtheorem{lemma}[lemmano]{Lemma}

\newenvironment{theorem*}{{\bf Lemma:}}

\title{Sinkhorn Based Associative Memory Retrieval Using Spherical Hellinger Kantorovich Dynamics}

\author{Aratrika Mustafi*}
\author{Soumya Mukherjee*}

\affil{Department of Statistics, Pennsylvania State University}

\newtheorem{asu}{Assumption}

\begin{document}

\maketitle

\begin{abstract}
  We propose a dense associative memory for empirical measures (weighted point clouds).
Stored patterns and queries are finitely supported probability measures, and retrieval is defined
by minimizing a Hopfield-style log-sum-exp energy built from the debiased Sinkhorn divergence.
We derive retrieval dynamics as a spherical Hellinger Kantorovich (SHK) gradient flow, which
updates both support locations and weights. Discretizing the flow yields
a deterministic algorithm that uses Sinkhorn potentials to compute barycentric transport steps
and a multiplicative simplex reweighting. Under local separation and PL-type conditions we prove
basin invariance, geometric convergence to a local minimizer, and a bound showing the minimizer
remains close to the corresponding stored pattern. Under a random
 pattern model, we further show that these Sinkhorn basins are disjoint
 with high probability, implying exponential capacity in the ambient 
dimension. Experiments on synthetic Gaussian point-cloud
memories demonstrate robust recovery from perturbed queries versus a Euclidean Hopfield-type baseline.
\end{abstract}

\section{Introduction}
Associative memories are a class of models designed to store and retrieve information through patterns of association instead of recalling data by specifying its location.  These models frame retrieval as a dynamical process :  given a partial or corrupted query, the system evolves toward an attractor that represents a stored pattern. The query evolves through a dynamical process that decreases an energy whose local minima encode the memories. This viewpoint goes back to Hopfield’s classical construction, where one designs recurrent interactions that induce a Lyapunov (energy) function whose local minima correspond to memories, yielding content-addressable recall from partial or corrupted cues (\cite{hopfield1982neural}). More recently, some modern Hopfield style models (\cite{krotov2016dense},\cite{ramsauer2020hopfield}) use the soft-min/ log-sum-exp (LSE) energy to create high-capacity retrieval dynamics and improve efficiency. In the continuous state setting, these retrieval updates can be related to the attention mechanism used in transformer architectures, suggesting a broad role for energy-based retrieval as a reusable computational primitive inside learned systems.

  A key limitation of much of the associative-memory literature is that patterns are typically vectors. A natural next step is to move beyond vector-valued patterns and treat probability measures as the objects being stored and retrieved. This shift is motivated by the growing role of distributional representations in modern learning, where uncertainty, multimodality, and population-level structure are often more naturally encoded by measures than by single points. In this direction, \cite{tankala2026denseassociativememorygaussian} develop a distributional dense associative memory for Gaussian distributions in the 2-Wasserstein (Bures-Wasserstein) geometry, defining a log-sum-exp energy over stored distributions and deriving a retrieval dynamics that aggregates optimal transport maps in a Gibbs-weighted manner; their stationary points correspond to self-consistent Wasserstein barycenters.  Many practical settings, however, are not well captured by a single Gaussian, instead distributions are often represented nonparametrically through weighted samples -  for example, weighted point clouds, histograms, particle approximations of posteriors, or sets of learned feature vectors,  suggesting the need for an associative-memory retrieval principle that operates directly on discrete measures.

In this paper we pursue this analogous extension of dense-associative-memory retrieval when both stored patterns and \textit{queries are empirical measures (weighted point clouds)}, using a log-sum-exp energy built from the debiased Sinkhorn divergence.  Sinkhorn divergences provide a computationally tractable way to compare discrete measures and yield deterministic barycentric projections from entropic optimal-transport couplings. A key issue is that transport only (Wasserstein) dynamics moves support locations while keeping the query weights fixed - an intrinsic limitation when the weights carry information such as saliency, mixture proportions, or coarse discretization effects. To evolve both support points and weights in a principled way while staying within probability measures, we therefore use a \textit{spherical Hellinger Kantorovich} (SHK) dynamics [see Appendix \ref{App:SHK flow}, \cite{liero2018optimal}], which couples Wasserstein (aka Kantorovich) transport with a spherical Hellinger-type reaction component. In the empirical setting, this leads to a fully deterministic retrieval operator that updates support points via Sinkhorn barycentric maps and updates weights via a multiplicative  reweighting on the simplex. This combination is attractive because it aligns the retrieval dynamics with the underlying geometry of measures while allowing both \textit{where mass} is located and \textit{how much mass} is assigned to each particle, to adapt during recall. With this objective in mind, our goal is to construct a dense associative memory (DAM) capable of storing discrete measures and accurately recovering the corresponding distribution when provided with perturbed/noisy queries.

Given a collection of finitely supported discrete distributions (stored patterns) $\{X_i\}_{i=1}^N$ and a query finitely supported discrete distribution $\xi$, we define the LSE functional using the Sinkhorn divergence $S_\varepsilon$ (defined in \ref{eq:sinkhorn-div}),  
\begin{equation}
\label{eq:energy}
E(\xi) \;=\; -\frac{1}{\beta}\log\!\left(\sum_{i=1}^N \exp\!\bigl(-\beta\,S_\varepsilon(\xi,X_i)\bigr)\right),
\end{equation}
with inverse temperature $\beta>0$. This is the functional we aim to minimize along the SHK gradient flow. To this end, we develop a practically implementable algorithm (SinkhornAlgo) to carry out the optimization. In parallel, we establish theoretical guarantees ensuring convergence and near-accurate retrieval, and we present empirical evidence to validate these findings.

\begin{itemize}
    \item \textbf{Local convergence and stability guarantees in SHK geometry -}
    Under standard local assumptions (margin separation, SHK smoothness, bounded gradients, and a
    local PL inequality), we prove descent and contraction properties for the softmin energy,
    quantify basin interference through separation margins, and establish geometric convergence of
    SHK gradient descent iterates to the unique local minimizer within a basin. We also bound the
    deviation between the basin minimizer and the corresponding stored pattern and provide a local
    stability bound showing stored patterns are approximate fixed points of one-step retrieval. 
    \item \textbf{Separation results under a random pattern model.}
  We introduce a sampling mechanism for generating general $M$-atom measures (random supports and
  weights) and prove that, with high probability, the induced Sinkhorn neighborhoods around
  stored patterns are pairwise disjoint. This yields an exponential-in-dimension scaling of the
  number of storable patterns under the stated conditions.
    \item \textbf{Empirical evidence against a Euclidean Hopfield-type baseline-}
    On synthetic point-cloud patterns sampled from Gaussians, we compare our SinkhornSHK retrieval
    against a vectorized Euclidean Hopfield-style baseline and observe robust recovery from noisy
    queries, particularly in regimes where Euclidean similarity is ambiguous but distributional
    (OT/Sinkhorn) geometry remains discriminative.
\end{itemize}

\section{Setting and notations}\label{Sec: Settings and Notations}

Let $\Omega\subset\mathbb{R}^d$ be open, bounded and convex with diameter $D\coloneqq\sup _{x, y \in \Omega}\|x-y\| < \infty$ and boundary $\partial \Omega$, and we equip it with the usual Euclidean geometry induced by the $\ell_2$ norm. For any $\omega \in \Omega$ and any set $S \subset \R^d$, define $\operatorname{dist}(\omega, S) = \min_{s \in S} \|\omega -s \|$. For any $x=(x_1,\dots,x_M) \in \Omega^M$, let $d_{\partial}(x) \coloneqq \min _{1 \leq m \leq M} \operatorname{dist}\left(x_m, \partial \Omega\right)$ denote the distance of $x$ from the boundary of $\Omega$ and $\operatorname{sep}(x)\coloneqq \min _{i \neq j}\left\|x_i-x_j\right\| $ denote the minimum pairwise separation between any two components of $x$. Let $\mathbb{B}(x, r) \coloneqq \left\{y \in \Omega : \|y-x\| < r \right\}$ and its closure $\bar{\mathbb{B}}(x, r) \coloneqq \left\{y \in \Omega : \|y-x\| \leq r \right\}$ be the open and closed ball of radius r around $x \in \Omega$. We denote by $\|\cdot\|_1$ and $\|\cdot\|_{\infty}$ the $\ell^1$ and $\ell^{\infty}$ norms, respectively and $\|\cdot\|$ will be interpreted in the natural context.

Fix any $M \in \mathbb{N}$, $\frac{1}{M} > \amin > 0$ and $\Dsep > 0$. Stored patterns are finitely supported discrete measures
$
X_i=\sum_{m=1}^{M} b_{i,m}\,\delta_{y_{i,m}},
b_{i,m}> \amin > 0, \sum_{m=1}^{M} b_{i,m}=1,
$
with all the locations $y_{i,m}$'s being separated from each other with minimum pairwise distance $\min _{m \neq k} \|y_{i,m}- y_{i,k}\| > \Dsep$,
and the query is also a finitely supported discrete measures
$\xi=\sum_{m=1}^{M} a_m\,\delta_{x_m},
a_m > \amin> 0, \sum_{m=1}^M a_m=1.
$ with all the locations $x_{m}$'s being pairwise separated with margin $\Dsep$.

For any $M \in \mathbb{N}$ $\frac{1}{M} >\amin > 0$ and $\Dsep > 0$, define the space of bounded weights
$
\aspace\coloneqq\left\{a \in \mathbb{R}^M: a_m > \amin, \sum_{m=1}^M a_m=1\right\} 
$ and the space of $M$ pairwise separated locations/atoms $\locspace(\Omega)\coloneqq\left\{x = (x_1,\dots,x_M) \in \Omega^M: \min _{m \neq n} \|x_m- x_n\| > \Dsep\right\} \in \aspace \times \locspace(\Omega)$. We will often consider the unordered particle parameterization of a finitely supported discrete measure with $M$ atoms $(a, x)=\left(\left(a_m\right)_{m=1}^M,\left(x_m\right)_{m=1}^M\right)$. Let $\mathcal{P}(\Omega)$ be the collection of probability measures on $\Omega$ and $$\ProbM = \left\{\xi \in \mathcal{P}(\Omega) : \xi = \sum_{l=1}^{M} a_l \delta_{x_l} \textrm{ for } (a_1,\dots,a_M) \in \aspace,  (x_1,\dots,x_M) \in \locspace(\Omega) \right\}.$$

Throughout the paper, we use the quadratic cost
$
c(x,y)=\frac12\|x-y\|^2
$
for defining the Wasserstein distance, spherical Hellinger Kantorovich distance and the Sinkhorn divergence.

For $\varepsilon>0$, define the entropic OT cost
$
\operatorname{OT}_\varepsilon(\mu,\nu)
\coloneqq\min_{\pi\in\Pi(\mu,\nu)}
\int_{\Omega\times\Omega} c(x,y)\,d\pi(x,y) \;+\;\varepsilon\,
\KLf{\pi}{\mu\otimes\nu},
$
where $\Pi(\mu,\nu)$ denotes couplings with marginals $\mu,\nu$.

The debiased \textit{Sinkhorn divergence} is defined as -
\begin{equation}
\label{eq:sinkhorn-div}
S_\varepsilon(\mu,\nu)
\coloneqq\operatorname{OT}_\varepsilon(\mu,\nu)
-\frac12\operatorname{OT}_\varepsilon(\mu,\mu)
-\frac12\operatorname{OT}_\varepsilon(\nu,\nu).
\end{equation}

A standard dual formulation of $\operatorname{OT}_\varepsilon$ uses Schr\"{o}dinger (entropic OT) potentials $(f_{\mu,\nu},g_{\mu,\nu})$, defined up to an additive constant.
We do not require a specific normalization; only gradients $\nabla f$ and differences of potentials will matter.
 One can define the entropic soft c-transform operator $A_\varepsilon$ via an expression of the form
\[
A_\varepsilon(g,\nu)(x)\coloneqq-\varepsilon\log\int_\Omega \exp\!\Big(\frac{g_{\mu,\nu}(y)-c(x,y)}{\varepsilon}\Big)\,d\nu(y)
\quad \text{(defined up to an additive constant)}.
\]
Then the optimal potentials $(f_{\mu,\nu},g_{\mu,\nu})$ (Schr\"{o}dinger potentials) satisfy the Schr\"{o}dinger system
\[
f_{\mu,\nu}=A_\varepsilon(g_{\mu,\nu},\nu)\quad \mu\text{-a.e.},
\qquad
g_{\mu,\nu}=A_\varepsilon(f_{\mu,\nu},\mu)\quad \nu\text{-a.e.}.
\]

These potentials are unique up to adding a constant to $f_{\mu,\nu}$ and subtracting the same constant from $g_{\mu,\nu}$ (gauge invariance). This does not affect any gradient $\nabla f_{\mu,\nu}(x)$ or $\nabla g_{\mu,\nu}(x)$, which is what we ultimately use.
The optimal entropic coupling has Gibbs form
\[
\frac{d\pi^\varepsilon_{\mu,\nu}}{d(\mu\otimes\nu)}(x,y)
=\exp\!\left(\frac{f_{\mu,\nu}(x)+g_{\mu,\nu}(y)-c(x,y)}{\varepsilon}\right).
\]
For more details on the Sinkhorn divergence, we refer  our readers to \cite{feydy2019interpolating,hardion2025gradient}.

\section{Theoretical Guarantees}
In this section, we state all our theoretical results. The detailed proofs are available in Section \ref{sec: Proof of theoretical results}. Retrieval guarantees largely rely on ensuring sufficient separation in terms of Sinkhorn divergence among the stored patterns. We propose a sampling mechanism that allows the number of patterns $N$ to be  exponentially large in the dimension $d$ and Theorem \ref{Exponential capacity and separation in Sinkhorn} establishes a high-probability separation guarantee for these $N$ randomly generated stored patterns, ensuring that the associated Sinkhorn neighborhoods (basins) are pairwise disjoint.

\begin{theorem}[Exponential storage capacity and high probability separation of patterns]\label{Exponential capacity and separation in Sinkhorn}
    Fix $d \geq 1, M \geq 2$, $\amin >0$, $\Dsep > 0$ and let $\Omega \subset \mathbb{R}^d$ be open, bounded and convex. Assume there exists $c \in \Omega$ and $R>0$ such that the closed ball $\bar{\mathbb{B}}(c, R) \subset \Omega$. Fix any $\sigma \in(0, R / 4)$, such that $\mathcal{Z}_{\sigma, \Dsep}\coloneqq\left\{\left(z_1, \ldots, z_M\right) \in \mathbb{B}(0, \sigma)^M: \min _{n \neq m}\left\|z_n-z_m\right\|>\Dsep\right\}$ is non-empty, and set $R_0\coloneqq R-2 \sigma$.
    Let $\gamma, p \in(0,1)$ and choose
    $$
    N\coloneqq\left\lfloor\sqrt{2 p} \exp \left(\frac{\gamma^2}{4} d\right)\right\rfloor .
    $$
    Let $X_1,\dots,X_N$ be generated by the sampling mechanism described as SampAlgo (see Sec \ref{SampAlgo}). Assume $\varepsilon>0$ is chosen small enough so that
$$
\varepsilon \log M<\frac{1-\gamma}{16} R_0^2,
$$
and define
$$
d_{\min }\coloneqq\sqrt{2(1-\gamma)} R_0, \quad r\coloneqq\frac{d_{\min }^2}{32}-\varepsilon \log M, \quad \Delta\coloneqq\frac{d_{\min }^2}{4} .
$$

Then, with probability atleast $1-p$, the following pairwise separation of Sinkhorn neighbourhoods/basins holds true (i.e. Assumption \ref{Ass: margin sep in Sinkhorn}):

For each $i$, for every $\xi \in B_i(r)=\left\{\nu \in \ProbM: S_{\varepsilon}\left(\nu, X_i\right) \leq r\right\}$, and every $j \neq i$,
$$
S_{\varepsilon}\left(\xi, X_j\right)-S_{\varepsilon}\left(\xi, X_i\right) \geq \Delta.
$$ In particular, $B_i(r) \cap B_i(j) = \emptyset, \quad \forall i\neq j$ i.e. the Sinkhorn neighborhoods of $X_i$'s are pairwise disjoint.
\end{theorem}

The Sinkhorn margin separation property provides the foundation required for reliable retrieval, which ensures that if a query is close to a stored pattern, it is separated enough from other patterns to ensure accurate retrieval. In fact, we can prove retrieval guarantees in much greater generality than afforded by the particular sampling algorithm we propose. We can establish these results under some regularity assumptions regarding the energy functional $E$ and the local Sinkhorn energies $F_i(\cdot) \coloneqq S_{\varepsilon}(\cdot,X_i)$, which are presented in Section \ref{sec: Assumptions}. 

Given the energy functional $E$ along with a query $\xi \in \ProbM$ and a step-size $\eta$, we can define a gradient based evolution of $E$ by equipping the space $\ProbM$ of finitely supported $M$-atom discrete probability measures with a choice of geometry. We choose the spherical Hellinger-Kantorovich (SHK) geometry and the gradient descent based one-step retrieval operator can be defined as
\begin{equation}\label{Gradient descent update}
\Phi_\eta(\xi)=\operatorname{Ret}_{\xi}(-\eta \operatorname{grad}_{\operatorname{SHK}}E(\xi))
\end{equation}
where $\operatorname{Ret}_{\xi}$ is a retraction map that evolves the query $\xi$ along the negative SHK gradient of E, given by $\operatorname{grad}_{\operatorname{SHK}}E(\xi)$, for a small step controlled by $\eta$ and ensures that the resulting object is still a probability distribution (in fact, an element of $\ProbM$). We relegate all details to the Appendix.

\begin{theorem}[Geometric convergence in Sinkhorn divergence to the local minimizer and local basin invariance of gradient descent iterates]\label{Geometric convergence in Sinkhorn divergence to the local minimizer}
Let Assumptions \ref{Ass: L-smoothness of energy functional E in SHK geometry}, \ref{Ass: Existence of minimizer of E in local basin} and \ref{Ass: PL inequality in local basin} hold. Define the stored pattern margins $w_i\coloneqq\min _{1 \leq m \leq M}\left(b_{i, m}-\amin\right)>0, d_i^{\partial}\coloneqq\min _{1 \leq m \leq M} \operatorname{dist}\left(y_{i, m}, \partial \Omega\right)>0$ and $s_i \coloneqq \operatorname{sep}\left(y_i\right)=\min _{m \neq n}\left\|y_{i, m}-y_{i, n}\right\|_2>\Dsep$. Let $\delta_i, \tau_i>0$ such that $0<\delta_i<\bar{\delta}_i\coloneqq\min \left\{d_i^{\partial}, \frac{s_i-\Dsep}{2}\right\}>0$ and $0<\tau_i<w_i$. Let $r>0$ be such that $0<r<\rloc_i\left(\delta_i, \tau_i\right)=\min \left\{\frac{\amin\delta_i^2}{2}-\varepsilon \log M, \frac{\tau_i\left(s_i-\delta_i\right)^2}{4}-\varepsilon \log M\right\}$. Let $E_i^{*}(r)$ be the minimum value of $E(\xi)$ in the local basin $B_i(r)$ and $X_i^{*}(r)$ be a minimizer of $E$ in $B_i(r)$ i.e. $E(X_i^{*}(r)) = E_i^{*}(r) \coloneqq \inf_{\xi \in B_i(r)} E(\xi)$. Finally, define $\etaretarg{i}\coloneqq \min \left\{\frac{\lambda^2}{2 D^2} \log \frac{\min _m b_{i, m}-\tau_i}{\amin}, \frac{1}{2 D} \min \left\{d_i^{\partial}-\delta_i, s_i-2 \delta_i-\Dsep\right\}\right\}$.

Then for any step size $0<\eta < \min \{\frac{1}{L},\frac{1}{\mu},\etaretarg{i}\}$, the following hold true:
\begin{enumerate}
        \item If $\xi^{(k)} \in B_i(r)$ for all $k \in \mathbb{N} \cup \{0\}$, then the sequence of SHK gradient iterates $\left(\xi^{(k)}\right)_{k \geq 0}$ satisfy the explicit geometric bound
        $$
        S_{\varepsilon}\left(\xi^{(k)}, X_i^*(r)\right) \leq \frac{G S_{\eta} \sqrt{2 \eta\left(E\left(\xi^{(0)}\right)-E_i^*(r)\right)}}{1-\left(1-\mu \eta\right)^{\frac{1}{2}}} \cdot (1-\eta \mu)^\frac{k}{2},
        $$
        where
        $S_\eta\leq \min\left\{e^\frac{\eta D^2}{\lambda^2},\frac{1}{\sqrt{\amin}}\right\}$ and $G = D \sqrt{1+\frac{D^2}{\lambda^2}}$ with $D = \sup_{x,y \in \Omega}\|x-y\|$ and $\lambda$ being the relative strength scale of the spherical Hellinger component.
        \item  If $\xi^{(k)} \in B_i(r)$ for all $k \in \mathbb{N} \cup \{0\}$, the sequence of SHK gradient iterates $\left(\xi^{(k)}\right)_{k \geq 0}$ converges weakly to $X_i^{*}(r)$ in $\mathcal{P}(\Omega)$.
        \item If $\xi^{(k)} \in B_i(r)$ for all $k \in \mathbb{N} \cup \{0\}$, then for any $\delta >0$, the error bound $S_{\varepsilon}\left(\xi^{(k)}, X_i^*(r)\right) \leq \delta$ is guaranteed to be achieved once the number of iterations $k$ is greater than or equal to $\frac{\min\left\{\frac{2\eta D^2}{\lambda^2}, - \log \amin\right\} + 2\log G+\log(2\eta \left(E\left(\xi^{(0)}\right)-E_i^{*}(r)\right)) + 2\log (\frac{1}{\delta(1- \sqrt{1- \mu\eta})})}{- \log (1-\mu \eta)}$. A simpler sufficient condition on the no. of iterations to achieve the same error bound is $k \geq \min\left\{\frac{2D^2}{\mu\lambda^2}, - \frac{\log \amin}{\mu\eta}\right\} + \frac{1}{\mu \eta} \log(2\eta \left(E\left(\xi^{(0)}\right)-E_i^{*}(r)\right)) + \frac{2}{\mu \eta}\log (\frac{2G}{\delta \mu \eta})$.
\end{enumerate}

In addition, for $\rho_i(\eta,r,\xi^{(0)})\coloneqq G\frac{\sqrt{2\eta\left(E\left(\xi^{(0)}\right)-E_i^{*}(r)\right)}}{\sqrt{\amin}(1-\sqrt{1-\mu \eta})}$ and $\alpha(r,\xi^{(0)}) \coloneqq r^2 \times\frac{\mu \amin}{2 G^2 \left(E\left(\xi^{(0)}\right)-E_i^{*}(r)\right)} $, if the step-size $\eta$  and the initial iterate $\xi^{(0)}$ satisfies the conditions $\alpha(r,\xi^{(0)}) >1$, $\frac{4 \alpha(r,\xi^{(0)})}{\mu(\alpha(r,\xi^{(0)})+1)^2}  \leq \eta < \min \{1 / L, 1 / \mu,\etaretarg{i}\}$ and $S_{\varepsilon}(\xi^{(0)},X_i) \leq r - \rho_i(\eta,r,\xi^{(0)})$ i.e. $\xi^{(0)} \in B_i(r - \rho_i)$, then the sequence of SHK gradient descent iterates $(\xi^{(k)})_{k \geq 0}$ all belong to the local basin $B_i(r)$. Consequently, the above 3 properties hold true without the apriori basin invariance assumption $\xi^{(k)} \in B_i(r)$ for all $k \in \mathbb{N} \cup \{0\}$.

\end{theorem}

Theorem \ref{Geometric convergence in Sinkhorn divergence to the local minimizer} provides the core algorithmic guarantee, showing that the SHK gradient descent dynamics of $E$ converges geometrically (in terms of no. of iterations) to the unique local minimizer within that basin and remain invariant inside it. Together, these results connect statistical separation with dynamical stability, yielding rigorous theoretical guarantees for accurate associative-memory retrieval.

\begin{theorem}[Sinkhorn distance between minimizer in local basin and stored pattern and stability of stored pattern]\label{Sinkhorn distance between minimizer in local basin and stored pattern and stability of stored pattern}
    Let Assumptions \ref{Ass: margin sep in Sinkhorn} and \ref{Ass: Existence of minimizer of E in local basin} hold true. Then, we have that

    \begin{enumerate}
            \item $S_{\varepsilon}(X_i^{*}(r),X_i) \leq \frac{1}{\beta} \log \left(1+(N-1) e^{-\beta \Delta}\right) \leq \frac{N-1}{\beta} e^{-\beta \Delta}.$

        \item for $\eta \leq \etaretarg{i}$ as defined in Theorem \ref{Geometric convergence in Sinkhorn divergence to the local minimizer}, $S_{\varepsilon}(X_i,\Phi_\eta(X_i)) 
   \leq \frac{\min\left\{e^{\eta D^2 / \lambda^2},\frac{1}{\sqrt{\amin}}\right\} \eta G^2(N-1) e^{-\beta \Delta}}{1+(N-1) e^{-\beta \Delta}}$.
    \end{enumerate}
    
\end{theorem}

While Theorem \ref{Geometric convergence in Sinkhorn divergence to the local minimizer} establishes geometric convergence of the SHK gradient descent iterates to the unique local minimizer within a Sinkhorn basin, it does not yet quantify how well this minimizer approximates the original stored pattern. Theorem \ref{Sinkhorn distance between minimizer in local basin and stored pattern and stability of stored pattern} closes this gap by showing that the basin minimizer remains exponentially close (in terms of the inverse temperature $\beta$ and margin separation $\Delta$) to the corresponding stored pattern in Sinkhorn divergence, and that each stored pattern is an approximate fixed point of the retrieval operator. Thus, beyond dynamical convergence, Theorem \ref{Sinkhorn distance between minimizer in local basin and stored pattern and stability of stored pattern} provides a fidelity guarantee: the attractor reached by the algorithm is not merely stable, but provably close to the intended memory, ensuring accurate and stable associative recall.

\section{Conclusion}
We developed a dense associative memory for empirical measures based on a Sinkhorn log-sum-exp energy and spherical Hellinger Kantorovich gradient dynamics, yielding deterministic transport-reaction retrieval with provable local convergence and separation guarantees. The full retrieval algorithm (pseudo-code and implementation
details) along with the numerical experiments is provided in the Appendix (see Sections \ref{Discrete retrieval algorithm for empirical measures} and \ref{Numerical Experiments}). Future directions include faster retrieval implementations, adaptive regularization and applying our algorithm on real point clouds.

\section*{References}
\bibliographystyle{plainnat}
\bibliography{ref.bib}
\nocite{*}

%%%%%%%%%%%%%%%%%%%%%%%%%%%%%%%%%%%%%%%%%%%%%%%%%%%%%%%%%%%%
\newpage

\appendix

\section*{Appendix : }

\section{First variation of the Sinkhorn divergence}

Let $F_\nu(\mu)\coloneqq S_\varepsilon(\mu,\nu)$ with $\nu$ fixed.
A first variation $\delta F_\nu/\delta\mu$ is defined (up to an additive constant) by
\[
\left.\frac{d}{dt}F_\nu(\mu+t\chi)\right|_{t=0}
=\int_\Omega \left(\frac{\delta F_\nu}{\delta\mu}(\mu)(x)\right)\,d\chi(x),
\qquad \int d\chi=0.
\]
It is standard to see that (derivation available in \cite{hardion2025gradient}) 
\begin{equation}
\label{eq:first-variation-sinkhorn}
\frac{\delta S_\varepsilon(\mu,\nu)}{\delta\mu}
= f_{\mu,\nu}-\tfrac12\bigl(f_{\mu,\mu}+g_{\mu,\mu}\bigr)
\qquad \text{in }C(\Omega)/\mathbb{R}.
\end{equation}
Intuitively, $\frac{\delta \operatorname{OT}_\varepsilon(\mu,\nu)}{\delta \mu} = f_{\mu,\nu}$ and $\frac{\delta \operatorname{OT}_\varepsilon(\mu,\nu)}{\delta \nu} = g_{\mu,\nu}$. Since $\mu$ appears in \emph{both} marginals of the self term $\operatorname{OT}_\varepsilon(\mu,\mu)$, the first variation of $\tfrac12\operatorname{OT}_\varepsilon(\mu,\mu)$ w.r.t.\ $\mu$ is the average potential
\[
f^{\mathrm{sym}}_{\mu,\mu}\coloneqq\tfrac12\bigl(f_{\mu,\mu}+g_{\mu,\mu}\bigr),
\]
which is invariant under the Sinkhorn gauge transformation $(f,g)\mapsto(f+c,g-c)$. When one chooses a symmetric gauge for the self problem (i.e. \ $f_{\mu,\mu}=g_{\mu,\mu}$), which is possible for symmetric costs),~\eqref{eq:first-variation-sinkhorn} reduces to the commonly stated formula $f_{\mu,\nu}-f_{\mu,\mu}$ in $C(\Omega)/\mathbb{R}$.

\section{First variation of the log-sum-exp energy}

Let $Z(\xi)=\sum_{i=1}^N \exp(-\beta S_\varepsilon(\xi,X_i))$ and define Gibbs weights
\begin{equation}\label{softmax weights of energy E}
w_i(\xi)\coloneqq\frac{\exp(-\beta S_\varepsilon(\xi,X_i))}{\sum_{j=1}^N \exp(-\beta S_\varepsilon(\xi,X_j))}.
\end{equation}
Differentiating~\eqref{eq:energy} and using~\eqref{eq:first-variation-sinkhorn} gives
\begin{equation}
\label{eq:first-variation-energy}
\frac{\delta E}{\delta\xi}(\xi)
=\sum_{i=1}^N w_i(\xi)\Bigl(f_{\xi,X_i}-\tfrac12\bigl(f_{\xi,\xi}+g_{\xi,\xi}\bigr)\Bigr)
=\left(\sum_{i=1}^N w_i(\xi)f_{\xi,X_i}\right)-\tfrac12\bigl(f_{\xi,\xi}+g_{\xi,\xi}\bigr),
\qquad \text{in }C(\Omega)/\mathbb{R}.
\end{equation}

We will denote the centered first variation of the energy functional $E(\cdot)$ as \[u_{\xi} \coloneqq \frac{\delta E}{\delta\xi}(\xi) - \Big\langle \frac{\delta E}{\delta\xi}(\xi), \xi \Big\rangle = \frac{\delta E}{\delta\xi}(\xi) - \int \frac{\delta E}{\delta\xi}(\xi) d \xi\] where $\langle \cdot , \cdot \rangle$ is the duality pairing.
\section{From entropic potentials to deterministic barycentric maps}\label{App: entropic potentials}

For quadratic cost $c(x,y)=\tfrac12\|x-y\|^2$, the gradient of a Schr\"{o}dinger potential has the barycentric form
\begin{equation}
\label{eq:grad-potential-bary}
\nabla f_{\mu,\nu}(x)=x-T^\varepsilon_{\mu\to\nu}(x),
\qquad
T^\varepsilon_{\mu\to\nu}(x)\coloneqq\int y\,\pi^\varepsilon_{\mu,\nu}(dy\mid x),
\end{equation}
where $\pi^\varepsilon_{\mu,\nu}(dy\mid x)$ is the conditional distribution under the optimal entropic coupling. The full derivation is provided in Section \ref{sec: Computing gradient of Schrödinger potentials explicitly for quadratic costs}. Thus, using~\eqref{eq:first-variation-energy}, the transport velocity becomes
\begin{equation}
\label{eq:velocity-bary}
v(x)=\sum_{i=1}^N w_i(\xi)\,T^\varepsilon_{\xi\to X_i}(x)-T^\varepsilon_{\xi\to\xi}(x).
\end{equation}
Even though~\eqref{eq:first-variation-energy} involves the symmetric self potential $\tfrac12(f_{\xi,\xi}+g_{\xi,\xi})$, the \emph{transport} velocity depends only on gradients. For symmetric costs (such as $\tfrac12\|x-y\|^2$) and a self-coupling, the optimal entropic plan is symmetric and one has $\nabla f_{\xi,\xi}=\nabla g_{\xi,\xi}$, so
\[
\nabla\tfrac12\bigl(f_{\xi,\xi}+g_{\xi,\xi}\bigr)=\nabla f_{\xi,\xi},
\]
and the self-correction in~\eqref{eq:velocity-bary} remains the usual barycentric map $T^\varepsilon_{\xi\to\xi}$.
This is a deterministic vector field on the query support computed from barycentric projections of Sinkhorn plans.

\section{Spherical Hellinger-Kantorovich gradient flow (Continuous-time)} \label{App:SHK flow}

\subsection{Transport$+$reaction based continuity equation}

We consider a transport-reaction dynamics of the form
\begin{equation}
\label{eq:transport-reaction}
\partial_t \xi_t + \nabla\cdot(\xi_t v_t) = \xi_t\,r_t,
\end{equation}
where $v_t$ is a velocity field (transport) and $r_t$ is a scalar reaction rate (mass reweighting).
A SHK gradient flow of $E$ sets
\begin{equation}
\label{eq:SHK-v-r}
v_t(x)=-\nabla\!\left(\frac{\delta E}{\delta\xi}(\xi_t)(x)\right) = -\nabla u_{\xi_t}(x),
\quad
r_t(x)=-\frac{1}{\lambda^2}\left(\frac{\delta E}{\delta\xi}(\xi_t)(x)-\left\langle \frac{\delta E}{\delta\xi}(\xi_t),\xi_t\right\rangle\right)=-\frac{1}{\lambda^2}u_{\xi_t}(x),
\end{equation}
with a scale parameter $\lambda>0$ controlling the relative strength of the spherical Hellinger component.
The subtraction of the mean ensures $\frac{d}{dt}\int d\xi_t=0$, so probability mass is preserved.

\subsection{Riemmanian structure induced by Spherical Hellinger-Kantorovich geometry on the space of probability measures}

Let $\xi \in \mathcal{P}(\Omega)$ where $\mathcal{P}(\Omega)$ is the set of all probability measures defined on $\Omega$ equipped with the Spherical Hellinger-Kantorovich geometry. A (sufficiently regular) tangent vector at $\xi$ can be represented by a pair $(r, v)$ consisting of a vector field $v: \Omega \rightarrow \mathbb{R}^d$ (transport velocity) and a scalar field $r: \Omega \rightarrow \mathbb{R}$ (reaction rate),
subject to the mass constraint $\int_{\Omega} r(x) d \xi(x)=0$

Given such $(v, r)$, the induced infinitesimal change of measure is the distribution $\dot{\xi}$ defined by the weak form
\begin{equation}\label{eq: weak form of SHK evolution PDE}
\frac{d}{d t} \int_{\Omega} \varphi d \xi_t=\int_{\Omega} \nabla \varphi(x) \cdot v(x) d \xi_t(x)+\int_{\Omega} \varphi(x) r(x) d \xi_t(x)
\end{equation}
which corresponds to the PDE
\begin{equation}\label{eq: SHK evolution PDE}
\partial_t \xi_t+\nabla \cdot\left(\xi_t v\right)=\xi_t r .
\end{equation}

Fix $\lambda>0$. Define the inner product on the tangent space at $\xi$ by
$$
\left\langle(r, v),\left(r^{\prime}, v^{\prime}\right)\right\rangle_{\operatorname{SHK},\xi}\coloneqq\int_{\Omega}\left(v \cdot v^{\prime}+\lambda^2 r r^{\prime}\right) d \xi
$$
for pairs satisfying $\int r d \xi=\int r^{\prime} d \xi=0$ and this induces the metric tensor at $\xi$ to be
$$g_{\xi}^{\SHK}\left((r, v),\left(r^{\prime}, v^{\prime}\right)\right) = \left\langle(r, v),\left(r^{\prime}, v^{\prime}\right)\right\rangle_{\operatorname{SHK},\xi}.
$$
The induced norm on the tangent space at $\xi$ is
$$
\|(r, v)\|_{\operatorname{SHK},\xi}^2=\int_{\Omega}\left(\|v\|_2^2+\lambda^2 r^2\right) d \xi.
$$

\subsection{The manifold of finitely supported discrete measures equipped with SHK geometry}

For any $M \in \mathbb{N}$, $\frac{1}{M} > \amin > 0$ and $\Dsep > 0$, define \\$\locspace(\Omega)\coloneqq\left\{x = (x_1,\dots,x_M) \in \Omega^M: \min _{m \neq n} \|x_m- x_n\| > \Dsep\right\} \in \aspace \times \locspace(\Omega)$ as the space of $M$ pairwise separated locations/atoms corresponding to the space of finitely supported discrete probability distributions with exactly $M$ atoms and let\\
$
\aspace\coloneqq\left\{a \in \mathbb{R}^M: a_m > \amin, \sum_{m=1}^M a_m=1\right\} 
$ be the space of bounded probability weights associated with the $M$ atoms. Further, define the associated parameter space (ordered particles) to be
$$
\mathcal{M}_M\coloneqq\aspace \times \locspace(\Omega) .
$$
For any $(a, x) \in \mathcal{M}_M$, we associate to it the discrete measure $\sum_{m=1}^M a_m \delta_{x_m}$ through the parametrization mapping $\Xi : \mathcal{M}_M \to \mathcal{P}(\Omega)\subset \left(C^{1}(\Omega)\right)^{*}$, defined as
$$
\Xi(a, x)\coloneqq\sum_{m=1}^M a_m \delta_{x_m} \in \mathcal{P}(\Omega) .
$$
Here $C^{k}(\Omega)$ is the class of $k$-times continuously differentiable functions on the domain $\Omega$, and $\left(C^{k}(\Omega)\right)^{*}$ represents the space of all continuous linear functionals on $C^{k}(\Omega)$ i.e. the dual space of $C^{k}(\Omega)$. Similarly, $C_c^{\infty}\left(\Omega\right)$ represents the the class of compactly supported infinitely differentiable functions on the domain $\Omega$, with $C_c^{\infty}\left(\Omega\right)$, with its corresponding continuous dual being $\left(C_c^{\infty}\left(\Omega\right)\right)^{*}$. Since all of our theory is focused on the discrete measures $\Xi(a, x)$ and functionals defined using such measures, and all such objects are invariant under permutations of the labels associated with the weight-location pairs, we will consider the quotient space $\widetilde{\mathcal{M}}_M = \mathcal{M}_M / \mathfrak{S}_M$ where $\mathfrak{S}_M$ is the symmetric group that acts freely on $\mathcal{M}_M$ by relabeling:
$$
\sigma \cdot\left(a_1, \ldots, a_M, x_1, \ldots, x_M\right)=\left(a_{\sigma^{-1}(1)}, \ldots, a_{\sigma^{-1}(M)}, x_{\sigma^{-1}(1)}, \ldots, x_{\sigma^{-1}(M)}\right) .
$$
Since the action is free and $\mathfrak{S}_M$ is finite, the quotient $\widetilde{\mathcal{M}}_M$ is a smooth finite-dimensional manifold, and it identifies canonically with the set of probability measures on $\Omega$ having exactly $M$ pairwise separated support points and positive weights.

For computations and developing the theory without introducing extra notational overhead, it is simplest to work on the ordered cover $\mathcal{M}_M$. The mathematical objects of interest are permutation invariant, so the theory developed on $\mathcal{M}_M$ descends to the quotient space $\widetilde{\mathcal{M}}_M$. Hence, from here on we identify $\widetilde{\mathcal{M}}_{M}$ with $\mathcal{M}_M$ itself, which is equivalent to treating $(a,x) \in \mathcal{M}_M$ as unordered tuples. We also identify $(a, x)$ with $\xi=\Xi(a, x)$ when convenient.

A tangent vector at $(a, x)$ is a pair $(\delta a, \delta x) \in \mathbb{R}^M \times\left(\mathbb{R}^d\right)^M$ with the simplex constraint $\sum_{m=1}^M \delta a_m=0 .$ Thus
$$
T_{(a, x)} \mathcal{M}_M=\left\{(\delta a, \delta x): \sum_m \delta a_m=0\right\}
$$

For a given $(a,x) \in \mathcal{M}_M$ and $(\delta a, \delta x) \in T_{(a, x)} \mathcal{M}_M $, we now proceed to compute the differential of $\Xi$ at $(a, x)$ in the direction $(\delta a, \delta x)$. Let $(a(t), x(t))$ be any $C^1$ curve in $\mathcal{M}_M$ such that
$$
(a(0), x(0))=(a, x) \,\ \textrm{ and }\,\ \frac{d}{dt}(a(t),x(t)) \Big \rvert_{t=0}=(\dot{a}(0), \dot{x}(0))=(\delta a, \delta x) .
$$
Let us define the induced curve of measures
$$
\xi_t\coloneqq\Xi(a(t), x(t))=\sum_{m=1}^M a_m(t) \delta _{x_{m(t)}}.
$$
For every test function $\varphi \in C^1(\Omega)$, we have that
$$
\int_{\Omega} \varphi d \xi_t=\sum_{m=1}^M a_m(t) \varphi\left(x_m(t)\right) .
$$
Differentiating at $t=0$, we obtain
$$
\begin{aligned}
\left.\frac{d}{d t}\right|_{t=0} \int_{\Omega} \varphi d \xi_t=&\sum_{m=1}^M \dot{a}_m(0) \varphi\left(x_m\right)+\sum_{m=1}^M a_m \nabla \varphi\left(x_m\right) \cdot \dot{x}_m(0)\\=&\sum_{m=1}^M \delta a_m \varphi\left(x_m\right)+\sum_{m=1}^M a_m \nabla \varphi\left(x_m\right) \cdot \delta x_m.
\end{aligned}
$$
Hence
$$
d \Xi_{(a, x)}(\delta a, \delta x)
$$
is the distribution characterized by
\begin{equation}\label{eq: differential of measure mapping}
\left\langle d \Xi_{(a, x)}(\delta a, \delta x), \varphi\right\rangle=\sum_{m=1}^M \delta a_m \varphi\left(x_m\right)+\sum_{m=1}^M a_m \nabla \varphi\left(x_m\right) \cdot \delta x_m .
\end{equation}
Equivalently,
$$
d \Xi_{(a, x)}(\delta a, \delta x)=\sum_{m=1}^M \delta a_m \delta_{x_m}-\nabla \cdot\left(\sum_{m=1}^M a_m \delta x_m \delta_{x_m}\right)
$$
in the sense of distributions.

Comparing Equation \ref{eq: differential of measure mapping} with the weak form of the SHK evolution PDE in Equation \ref{eq: weak form of SHK evolution PDE}, we can uniquely represent the tangent vector (distribution) $d \Xi_{(a, x)}(\delta a, \delta x)$ using any pair $(r,v)$ such that $r(x_m) = \frac{\delta a_m}{a_m}$ and $v(x_m) = \delta x_m$. The mass constraint is automatically satisfied since $\int_{\Omega} r d \xi=\sum_{m=1}^M a_m \frac{\delta a_m}{a_m}=\sum_{m=1}^M \delta a_m=0$. The uniqueness on the support points $x_1,\dots,x_M$ can be established using Lemma \ref{Linear independence of Dirac distribution and its distributional derivative}.

Now, we define the pullback metric tensor on $\mathcal{M}_M$ at $(a,x)$ induced by $g_{\Xi(a,x)}^{\SHK}$ as
$$
g_{(a,x)}\coloneqq \Xi^* g_{\Xi(a,x)}^{\SHK},
$$
which means that, for $\zeta=(\delta a, \delta x)$ and $\eta=\left(\delta a^{\prime}, \delta x^{\prime}\right)$ belonging to $ T_{(a, x)} \mathcal{M}_M$,
$$
g_{(a, x)}(\zeta, \eta)=g_{\xi}^{\SHK}\left(d \Xi_{(a, x)} \zeta, d \Xi_{(a, x)} \eta\right), \quad \xi=\Xi(a, x) .
$$
Using the identification above, we have that
$$
r\left(x_m\right)=\frac{\delta a_m}{a_m}, \quad v\left(x_m\right)=\delta x_m, \quad r^{\prime}\left(x_m\right)=\frac{\delta a_m^{\prime}}{a_m}, \quad v^{\prime}\left(x_m\right)=\delta x_m^{\prime} .
$$
Therefore, 
$$
\begin{aligned}
g_{(a, x)}\left((\delta a, \delta x),\left(\delta a^{\prime}, \delta x^{\prime}\right)\right) & =\sum_{m=1}^M a_m\left(\delta x_m \cdot \delta x_m^{\prime}+\lambda^2 \frac{\delta a_m}{a_m} \frac{\delta a_m^{\prime}}{a_m}\right) \\
& =\sum_{m=1}^M\left(a_m \delta x_m \cdot \delta x_m^{\prime}+\frac{\lambda^2}{a_m} \delta a_m \delta a_m^{\prime}\right) .
\end{aligned}
$$
Consequently, the SHK-induced Riemannian structure on $\mathcal{M}_M$ can be described as follows. Define for $(\delta a, \delta x),\left(\delta a^{\prime}, \delta x^{\prime}\right) \in T_{(a, x)} \mathcal{M}_M$ :
$$
\left\langle(\delta a, \delta x),\left(\delta a^{\prime}, \delta x^{\prime}\right)\right\rangle_{(a, x)}\coloneqq g_{(a, x)}\left((\delta a, \delta x),\left(\delta a^{\prime}, \delta x^{\prime}\right)\right) =\sum_{m=1}^M\left(\frac{\lambda^2}{a_m} \delta a_m \delta a_m^{\prime}+a_m \delta x_m \cdot \delta x_m^{\prime}\right) .
$$
The corresponding norm is
$$
\|(\delta a, \delta x)\|_{(a, x)}^2=\sum_{m=1}^M\left(\frac{\lambda^2}{a_m}\left(\delta a_m\right)^2+a_m\left\|\delta x_m\right\|_2^2\right) .
$$

\subsection{Defining an appropriate retraction map for $\mathcal{M}_M$}\label{sec: defining retraction map}

A local retraction is the standard way to define a \say{first-order accurate exponential map} used in Riemannian gradient descent. Let $T \mathcal{M}_M$ denote the tangent bundle corresponding to the manifold $\mathcal{M}_M$, given by $T \mathcal{M}_M \coloneqq \sqcup _{(a,x) \in \mathcal{M}_M} T_{(a,x)}\mathcal{M}_M = \left\{((a,x),(\delta a, \delta x)) : (a,x) \in \mathcal{M}_M, (\delta a, \delta x) \in T_{(a,x)}\mathcal{M}_M\right\}$, where $T_{(a,x)}\mathcal{M}_M$ is the tangent space at the parameter $(a,x) \in \mathcal{M}_M$. Following the general definition (and consistent with retractions are used in particle evolutions for gradient descent, for e.g. see \cite{chizat2022sparse}) we use the following definition of a retraction map:

A smooth map Ret : $T \mathcal{M}_M \rightarrow \mathcal{M}_M$ is a retraction if for each base point $z=(a, x) \in \mathcal{M}_M$, its restriction $\operatorname{Ret}_z$ :
$T_z \mathcal{M}_M \rightarrow \mathcal{M}_M$ satisfies
(i) $\operatorname{Ret}_z(0)=z$,
(ii) $\operatorname{DRet}_z(0)=\operatorname{Id}$ on $T_z \mathcal{M}_M$. It need not be well-defined everywhere but there exists an open set $U_{(a,x)} \subset T_{(a,x)} \mathcal{M}_M$ containing the origin on which it is well-defined.

For our purpose, we will build Ret as a product of a retraction on $\aspace$ and one on $\locspace(\Omega)$.

For the retraction map on $\locspace(\Omega)$, define:
$$
\operatorname{Ret}_x^{\mathrm{pos}}(\delta x)\coloneqq x+\delta x, \quad \text { i.e. } \quad \operatorname{Ret}_{\left(x_1, \ldots, x_M\right)}^{\mathrm{pos}}\left(\delta x_1, \ldots, \delta x_M\right)=\left(x_1+\delta x_1, \ldots, x_M+\delta x_M\right) .
$$
and $\operatorname{Ret}_x^{\mathrm{pos}}(\cdot)$ is well-defined on $U_{x} = \left\{ \delta x \in T_{x} \Loc_{M}(\Omega) :\|\delta x\|_{\infty}<\frac{1}{2} \min \left(d_{\partial}(x), \operatorname{sep}(x)-\Dsep\right)\right\} \subset T_{x} \Loc_{M}$. Clearly $\operatorname{Ret}_x^{\mathrm{pos}}(0)=x$ and $\operatorname{DRet}_x^{\mathrm{pos}}(0)=\mathrm{Id}$. Therefore, this is a retraction on $\locspace(\Omega)$.

Let $a \in \aspace$. Its tangent space is
$$
T_a \aspace=\left\{\delta a \in \mathbb{R}^M: \sum_m \delta a_m=0\right\}.
$$
Define, for $\delta a \in T_a \aspace$, the retraction map
\begin{equation}
\operatorname{Ret}_a^{\mathrm{w}}(\delta a)\coloneqq\frac{a \odot \exp (\delta a / a)}{\sum_{j=1}^M a_j \exp \left(\delta a_j / a_j\right)} 
\end{equation}
where $(\delta a / a)_m\coloneqq\delta a_m / a_m$,
$\exp$ acts coordinatewise and $\odot$ is coordinatewise product. This is well-defined for all $\delta a \in U_a$ where $U_a=\left\{\delta a \in T_a \aspace:\left\|\frac{\delta a}{a}\right\|_{\infty}<\frac{1}{2} \min _m \log \frac{a_m}{\amin}\right\} .$. Further, it can be shown that $\operatorname{Ret}_a^{\mathrm{w}}(\delta a)$ is indeed a retraction map for $\aspace$.

Finally, the product retraction on $\mathcal{M}_M$ is defined as
$$\label{Retraction map}
\operatorname{Ret}_{(a, x)}(\delta a, \delta x)\equiv \operatorname{Ret}_{\Xi(a, x)}(\delta a, \delta x)\coloneqq\left(\operatorname{Ret}_a^{\mathrm{w}}(\delta a), \operatorname{Ret}_x^{\mathrm{pos}}(\delta x)\right) 
$$
which is a retraction on $U_{(a,x)} = \left\{(\delta a, \delta x) \in T_{(a,x)} \mathcal{M}_M : \delta a \in U_a \textrm{ and } \delta x \in U_x\right\} \subset T_{(a,x)} \mathcal{M}_M$.

\subsection{Riemannian gradient of the energy functional $E$ on $\mathcal{M}_M$}

Let $\xi=\Xi(a, x)$. For discrete $\xi$, write $u_m\coloneqq u_{\xi}\left(x_m\right)$.

\paragraph{Differential of $E$ in particle coordinates :} Consider a tangent perturbation $(\delta a, \delta x)$. Assume that the functional $E$ defined on $\mathcal{P}(\Omega)$ admits a $C^1(\Omega)$ first variation $\frac{\delta E}{\delta \xi}(\xi)$ for distributions of interest $\xi \in \mathcal{P}(\Omega)$, whose centered version is $u_{\xi} \coloneqq \frac{\delta E}{\delta \xi}(\xi) - \langle \frac{\delta E}{\delta \xi}(\xi),\xi\rangle = \frac{\delta E}{\delta \xi}(\xi) -\int_{\Omega} \frac{\delta E}{\delta \xi}(\xi) d\xi $. Now, for a given $(a,x) \in \mathcal{M}_M$ and $(\delta a, \delta x) \in T_{(a, x)} \mathcal{M}_M $, $(a(t), x(t))$ be any $C^1$ curve in $\mathcal{M}_M$ such that
$$
(a(0), x(0))=(a, x) \,\ \textrm{ and }\,\ \frac{d}{dt}(a(t),x(t)) \Big \rvert_{t=0}=(\dot{a}(0), \dot{x}(0))=(\delta a, \delta x) .
$$ and let the induced curve of measures be
$$
\xi_t\coloneqq\Xi(a(t), x(t))=\sum_{m=1}^M a_m(t) \delta _{x_{m(t)}}.
$$
Along the curve of measures $\xi_t$, by the chain rule, we have that the differential of $E \circ \Xi$ is given by 
\begin{equation}\label{eq: differential of E composed with measure mapping}
\begin{aligned}
    d \left(E \circ \Xi\right)_{(a,x)}\left[\delta a, \delta x\right] =& \frac{d}{d t}\left(E \circ \Xi\right)\left(a(t),x(t)\right)\Big\lvert_{t=0} \\
    =& \frac{d}{d t} E\left(\xi_t\right)\Big\lvert_{t=0}\\
    =&\left\langle\frac{\delta E}{\delta \xi}(\xi_0), \dot{\xi}_0\right\rangle \\
    =& \left\langle\frac{\delta E}{\delta \xi}(\xi_0), d \Xi(a,x)(\delta a, \delta x)\right\rangle\\
    =& \left\langle \frac{\delta E}{\delta \xi}(\xi_0),\sum_{m=1}^M \delta a_m \delta_{x_m}-\nabla \cdot\left(\sum_{m=1}^M a_m \delta x_m \delta_{x_m}\right)\right\rangle\\
    =& \sum_{m=1}^M \frac{\delta E}{\delta \xi}(\xi_0)\left(x_m\right) \delta a_m+\sum_{m=1}^M a_m \nabla \frac{\delta E}{\delta \xi}(\xi_0)\left(x_m\right) \cdot \delta x_m ..
\end{aligned}
\end{equation}

We note that, since $\sum_{m=1}^{M} \delta_m =0$, one can replace $\frac{\delta E}{\delta \xi}(\xi_0)$ by $\frac{\delta E}{\delta \xi}(\xi_0) + c$ in Equation \ref{eq: differential of E composed with measure mapping} for any constant $c$ without changing the result, which is often termed as gauge-invariance. In particular,
\begin{equation}\label{eq: differential of E composed with measure mapping using centered first variation}
    d \left(E \circ \Xi\right)_{(a,x)}\left[\delta a, \delta x\right] = \sum_{m=1}^M u_{\xi}\left(x_m\right) \delta a_m+\sum_{m=1}^M a_m \nabla u_{\xi}\left(x_m\right) \cdot \delta x_m .
\end{equation}

We will often abuse notation by using the shorthand $\xi \equiv \Xi[a,x]$ and treating $E$ as a functional directly over $\mathcal{M}_M$, in which case we will represent the differential of $E$ as
$$
d E(\xi)[\delta a, \delta x]=\sum_{m=1}^M u_{\xi}\left(x_m\right) \delta a_m+\sum_{m=1}^M a_m \nabla u_{\xi}\left(x_m\right) \cdot \delta x_m
$$

\paragraph{Riemannian gradient of E :} The Riemannian gradient $\operatorname{grad}_{\operatorname{SHK}} E(a, x) \in T_{(a, x)} \mathcal{M}_M$ is the unique tangent vector $\left(\beta, \zeta\right)$ such that for all $(\delta a, \delta x)$,

$$
g_{(a,x)} \left(\left(\beta, \zeta\right),(\delta a, \delta x)\right)=\left\langle\left(\beta, \zeta\right),(\delta a, \delta x)\right\rangle_{(a, x)}=d E(\xi)[\delta a, \delta x].
$$
With $\xi \equiv \Xi (a,x)$, define $t_{\xi,m} = \frac{\delta E}{\delta \xi}(\xi)(x_m) + c$ for any constant $c \in \R$. Then, we have that
$$
\sum_m\left(\frac{\lambda^2}{a_m} \beta_{m} \delta a_m+a_m \zeta_{m} \cdot \delta x_m\right)=\sum_m\left(t_{\xi,m} \delta a_m+a_m \nabla t_{\xi,m} \cdot \delta x_m\right) .
$$

Matching the $\delta x_m$ terms gives

$$
\zeta_{m}=\nabla t_{\xi,m}=\nabla u_{\xi}\left(x_m\right) .
$$

For weights, note $\delta a$ is constrained by $\sum_m \delta a_m=0$. The identity

$$
\sum_m \left(\frac{\lambda^2}{a_m} \beta_{m} - t_{\xi,m} \right) \delta a_m=0 \quad \forall \delta a: \sum_m \delta a_m=0
$$
holds if and only if the coefficients
$$
k_m\coloneqq\frac{\lambda^2}{a_m} \beta_m-t_{\xi,m}
$$
are all equal to the same constant $k$, which can be verified by taking $\delta a=e_i-e_j$ for all $i \neq j$. Hence
$$
\frac{\lambda^2}{a_m} \beta_m=t_{\xi,m}+k .
$$
Since, we must have $\sum_{m=1}^{M} \beta_m = 0$, we must have that
$$
0=\sum_{m=1}^M \beta_m=\frac{1}{\lambda^2} \sum_{m=1}^M a_m\left(t_{\xi,m}+k\right)=\frac{1}{\lambda^2}(\bar{t}+k),
$$
where
$$
\bar{t}\coloneqq\sum_{m=1}^M a_m t_{\xi,m} .
$$
Therefore $k=-\bar{t}$, so
$$
\beta_m=\frac{a_m}{\lambda^2}\left(t_{\xi,m}-\bar{t}\right) =\frac{a_m}{\lambda^2}u_m =\frac{a_m}{\lambda^2} u_{\xi}(x_m).
$$
So the Riemannian gradient of $E$ is
$$
\operatorname{grad}_{\operatorname{SHK}} E(a, x)=\left(\left(\frac{a_m}{\lambda^2}u_m\right)_{m=1}^M,\left(\nabla u_m\right)_{m=1}^M \right).
$$
\subsection{SHK gradient descent updates in terms of retraction maps}

Fix a step size $\eta>0$. The retraction-based SHK gradient descent update is as follows. Given $\xi^{(k)}=\Xi\left(a^{(k)}, x^{(k)}\right)$, define
$$
\left(a^{(k+1)}, x^{(k+1)}\right)\coloneqq\operatorname{Ret}_{\left(a^{(k)}, x^{(k)}\right)}\left(-\eta \operatorname{grad}_{\operatorname{SHK}} E\left(a^{(k)}, x^{(k)}\right)\right)=\operatorname{Ret}_{\left(a^{(k)}, x^{(k)}\right)}\left(-\eta \operatorname{grad}_{\operatorname{SHK}} E\left(\xi^{(k)}\right)\right),
$$
and set
$$
\xi^{(k+1)}\coloneqq\Xi\left(a^{(k+1)}, x^{(k+1)}\right)
$$
We can express the SHK gradient descent update using the operator
$\Phi_\eta\left(\xi^{(k)}\right)=\xi^{(k+1)}$.

For a measurable map $T: \Omega \rightarrow \Omega$, the pushforward $T_{\#} \xi$ is defined by
$$
\left(T_{\#} \xi\right)(A)=\xi\left(T^{-1}(A)\right), \quad \text { equivalently } \quad \int \varphi d\left(T_{\#} \xi\right)=\int \varphi \circ T d \xi.
$$
For a measurable function $\rho: \Omega \rightarrow(0, \infty)$, the reweighted measure $\rho \xi$ is defined by
$$
(\rho \xi)(A)\coloneqq\int_A \rho(x) d \xi(x).
$$
For a given $\xi$, define the pushforward map
$$
T_\eta^{\xi}(x)\coloneqq x-\eta \nabla u_{\xi}(x) .
$$
and the reweighting map
$$\rho_\eta^{\xi}(x)\coloneqq\frac{\exp \left(-\frac{\eta}{\lambda^2} u_{\xi}(x)\right)}{\int_{\Omega} \exp \left(-\frac{\eta}{\lambda^2} u_{\xi}(z)\right) d \xi(z)}.$$
Then the SHK gradient descent update is
$$
\Phi_\eta(\xi)=\left(T_\eta^{\xi}\right)_{\#}\left(\rho_\eta^{\xi} \xi\right) .
$$

\subsection{Particle (empirical-measure) dynamics}

Let $\xi_t=\sum_{m=1}^M a_m(t)\delta_{x_m(t)}$.
Plugging this ansatz into~\eqref{eq:transport-reaction} yields the Lagrangian system
\begin{equation}
\label{eq:particle-ode}
\dot x_m(t)=v_t(x_m(t)),
\qquad
\dot a_m(t)=a_m(t)\,r_t(x_m(t)),
\qquad m=1,\dots,M.
\end{equation}
Define the particle-wise first-variation values
\begin{equation}
\label{eq:z-def}
z_m \;\coloneqq\;\left(\frac{\delta E}{\delta\xi}(\xi)\right)(x_m)
=\sum_{i=1}^N w_i(\xi)\Bigl(f_{\xi,X_i}(x_m)-\tfrac12\bigl(f_{\xi,\xi}(x_m)+g_{\xi,\xi}(x_m)\bigr)\Bigr),
\end{equation}
and their $\xi$-average $\bar z\coloneqq\sum_{m=1}^M a_m z_m$.
Then~\eqref{eq:SHK-v-r} gives the weight dynamics
\begin{equation}
\label{eq:replicator-ode}
\dot a_m(t)=-\frac{1}{\lambda^2}\,a_m(t)\,\bigl(z_m-\bar z\bigr),
\qquad \sum_m a_m(t)=1.
\end{equation}
Equation~\eqref{eq:replicator-ode} is the (mass-preserving) \emph{replicator} equation and can be viewed as a spherical Hellinger / natural-gradient flow on the simplex.

\section{Assumptions for analysis of the Sinkhorn and SHK gradient flow based retrieval}\label{sec: Assumptions}

In this section, we list down some useful assumptions are useful to prove certain properties of the Spherical Hellinger-Kantorovich gradient descent. Define $F_i(\xi) = S_{\varepsilon}(\xi,X_i)$ for $=1,\dots,N$.

\begin{asu}[Margin separation]
    \label{Ass: margin sep in Sinkhorn} The stored patterns $X_1,\dots,X_N$ and $\varepsilon$ are such that there exists a radius $r>0$ and a corresponding $\Delta > 0$ such that for any\\ $\xi \in B_i(r)\coloneqq \left\{\mu \in \ProbM: S_{\varepsilon}(\mu,X_i) \leq r \right\}$, we have that, for all $j \neq i$,
    \begin{equation}\label{eq: margin sep in Sinkhorn}
    F_j(\xi) - F_{i}(\xi) \geq \Delta.
    \end{equation}
\end{asu}

\begin{asu}[Local Retraction L-smoothness of energy functional E in SHK geometry]\label{Ass: L-smoothness of energy functional E in SHK geometry} 
For the same choice of $r$ as in Assumption \ref{Ass: margin sep in Sinkhorn}, there exists $L>0$ such that for every $\xi \in B_i(r)$, every tangent vector $(w,v) \in T_{\xi}(w,v)$ and every $\eta>0$ small enough so that $\operatorname{Ret}_{\xi}(\eta(w, v))$ is well-defined, we have

\begin{equation}\label{eq: L-smoothness of energy functional E in SHK geometry}
E\left(\operatorname{Ret}_{\xi}(\eta(r, v))\right) \leq E(\xi)+\eta\langle\operatorname{grad}_{\operatorname{SHK}} E(\xi),(r, v)\rangle_{\operatorname{SHK},\xi}+\frac{L \eta^2}{2}\|(r, v)\|_{\operatorname{SHK},\xi}^2
\end{equation}

where $\langle\cdot, \cdot\rangle_{\operatorname{SHK},\xi}$ and $\|\cdot\|_{\operatorname{SHK},\xi}$ is the SHK inner product and norm respectively .
\end{asu}

Let $E_i^{*}(r) \coloneqq \inf_{\xi \in \ B_i(r)} E(\xi)$. We will often denote $E_i^{*}(r)$ using $E_i^{*}$ for convenience. 

\begin{asu}[Existence and uniqueness of minimizer of E in local basin]\label{Ass: Existence of minimizer of E in local basin}
For the same choice of $r$ as in Assumption \ref{Ass: margin sep in Sinkhorn}, there exists a minimizer $X_i^{*}(r) \in B_i(r)$  of E i.e. $E(X_i^{*}(r)) = E_i^{*}(r) \coloneqq \inf_{\xi \in B_i(r)} E(\xi)$. Further the minimizer $X_i^{*}(r)$ is unique.
\end{asu}

\begin{asu}[PL inequality in local basin] \label{Ass: PL inequality in local basin} For the same choice of $r$ as in Assumption \ref{Ass: margin sep in Sinkhorn} and conditional on Assumption \ref{Ass: Existence of minimizer of E in local basin} being true, there exists $\mu>0$ such that for every $\xi \in B_i(r)$, 

\begin{equation}\label{eq: PL inequality in local basin}
\frac{1}{2}\|\operatorname{grad}_{\operatorname{SHK}} E(\xi)\|_{\operatorname{SHK},\xi}^2 \geq \mu\left(E(\xi)-E_i^{*}(r)\right)
\end{equation}
\end{asu}

\section{Sampling Algorithm to ensure high probability separation of measures in $\ProbM$}\label{SampAlgo}

We construct a sampling model that (i) produces fully general $M$-atom measures (random weights, random supports) that belong to $\ProbM$, and (ii) ensures pattern separation with high probability for number of patterns $N$ exponentially large in $d$. For simplicity and concreteness, we assume that the domain $\Omega$ is such that there exists some point $c \in \Omega$ and some $R>0$ such that the Euclidean ball
$$
\bar{\mathbb{B}}(c, R)\coloneqq\{x \in \Omega :\|x-c\| \leq R\} \subset \Omega .
$$

\paragraph{Sampling algorithm (SampAlgo) :}\label{Sampling Algorithm} Consider the domain $\Omega$ such that $\mathbb{B}(c, R)\coloneqq\{x:\|x-c\| \leq R\} \subseteq \Omega  \subset \mathbb{R}^d$. Consider the shape radius $\sigma \in (0 , \frac{R}{4})$, margin parameter $\gamma \in (0,1)$ and inradius parameter $R_0 = R - 2\sigma$. Fix any $\varepsilon>0$ such that $\varepsilon \log M<\frac{1-\gamma}{16} R_0^2$. For each pattern $i=1,\dots,N$, 
\begin{enumerate}
    \item\textbf{Generate random mean using Rademacher random variables :} Sample $s_i \in\{ \pm 1\}^d$ with i.i.d. coordinates, $\mathbb{P}\left(s_{i, k}=+1\right)=\mathbb{P}\left(s_{i, k}=-1\right)=1 / 2$. Define $\mu_i \coloneqq c+\frac{R_0}{\sqrt{d}} s_i$.
    \item\textbf{Generate random weights :} Sample $b_i=\left(b_{i, 1}, \ldots, b_{i, M}\right)$ in i.i.d manner using any probability distribution supported on $\aspace$ for some fixed $\amin>0$. 
    \item\textbf{Sample random point cloud around means :} Sample $z_i = (z_{i, 1}, \ldots, z_{i, M}) \in (\R^d)^M$ from any probability distribution supported in\\ $\mathcal{Z}_{\sigma, \Dsep}\coloneqq\left\{\left(z_1, \ldots, z_M\right) \in \mathbb{B}(0, \sigma)^M: \min _{n \neq m}\left\|z_n-z_m\right\|>\Dsep\right\}$. 
    \item\textbf{Set means using mean correction and define support points:} Compute the weighted shape mean
    $\bar{z}_i\coloneqq\sum_{m=1}^M b_{i, m} z_{i, m}$ and define support points $y_{i, m}\coloneqq\mu_i+\left(z_{i, m}-\bar{z}_i\right), \quad m=1, \ldots, M$.
    \item\textbf{Define the discrete measure :} Set the pattern as the discrete $X_i\coloneqq\sum_{m=1}^M b_{i, m} \delta_{y_{i, m}} \in \ProbM$
\end{enumerate}

\section{Proof of main theoretical results}\label{sec: Proof of theoretical results}

\subsection{Proof of Theorem \ref{Exponential capacity and separation in Sinkhorn}}

\begin{proof}    
    We first verify that the probability distributions $X_1,\dots,X_N$ generated by the sampling algorithm SampAlgo (see Sec \ref{SampAlgo}) belong to $\ProbM$.  First, by construction of Step 2 in SampAlgo (see Sec \ref{SampAlgo}), we have that, for $i=1,\dots,N$, $b_i \in \aspace$. Next, we prove strict pairwise separation of the support points $y_{i, m}$. For $m \neq n$,
    $$
    y_{i, n}-y_{i, m}=\left(\mu_i+z_{i, n}-\bar{z}_i\right)-\left(\mu_i+z_{i, m}-\bar{z}_i\right)=z_{i, n}-z_{i, m}.
    $$
    and hence $\left\|y_{i, n}-y_{i, m}\right\|=\left\|z_{i, n}-z_{i, m}\right\|$  for $i=1,\dots,N$. Since $z_i \in \mathcal{Z}_{\sigma, \Dsep}$, we have that $\min _{m \neq n}\left\|z_{i, n}-z_{i, m}\right\|>\Dsep$. Therefore, we have that
    $$
    \min _{m \neq n}\left\|y_{i, n}-y_{i, m}\right\|>\Dsep .
    $$
    Finally, we prove that every support point $y_{i, m}$ lies in $\Omega$. Since each $z_{i, m} \in \bar{\mathbb{B}}(0, \sigma)$, we have that $\left\|z_{i, m}\right\| \leq \sigma$. Further, since $\bar{z}_i=\sum_{m=1}^{M} b_{i, m} z_{i, m}$ is a convex combination of the $z_{i, m}$, we have that
    $$
    \left\|\bar{z}_i\right\| \leq \sum_{m=1}^M b_{i, m}\left\|z_{i, m}\right\| \leq \sum_{m=1}^M b_{i, m} \sigma=\sigma .
    $$
    Further,
    $$
    \left\|\mu_i-c\right\|=\left\|\frac{R_0}{\sqrt{d}} s_i\right\|=\frac{R_0}{\sqrt{d}}\left\|s_i\right\|=\frac{R_0}{\sqrt{d}} \sqrt{d}=R_0
    $$
    Hence, for each $m$, we have that
    $$
    \left\|y_{i, m}-c\right\| \leq\left\|\mu_i-c\right\|+\left\|z_{i, m}\right\|+\left\|\bar{z}_i\right\| \leq R_0+\sigma+\sigma=R .
    $$
    Therefore, we have that 
    $$
    y_{i, m} \in \bar{\mathbb{B}}(c, R) \subset \Omega.
    $$
    Therefore, for $i=1,\dots,N$, the generated pattern $X_i$ indeed belongs to $\ProbM$. Now, we compute its mean:
    $$
    \begin{aligned}
    m\left(X_i\right)=\sum_{m=1}^M b_{i, m} y_{i, m}=&\sum_{m=1}^M b_{i, m}\left(\mu_i+z_{i, m}-\bar{z}_i\right) \\
    =& \mu_i+\sum_{m=1}^M b_{i, m} z_{i, m}-\sum_{m=1}^M b_{i, m} \bar{z}_i\\
    =&\mu_i+\bar{z}_i-\bar{z}_i=\mu_i .
    \end{aligned}
    $$
    Now, define the event
    $$
    A\coloneqq\left\{\forall 1 \leq i<j \leq N,\left\|\mu_i-\mu_j\right\| \geq d_{\min }=\sqrt{2(1-\gamma)} R_0\right\} .
    $$
    choosing $N\coloneqq\left\lfloor\sqrt{2 p} \exp \left(\frac{\gamma^2}{4} d\right)\right\rfloor$, we have that $\binom{N}{2} \exp\left(-\frac{\gamma^2 d}{2} \right) =\frac{N(N-1)}{2} \exp\left(-\frac{\gamma^2 d}{2} \right) \leq p$. Therefore, by Lemma \ref{uniform separation for all pairs of means of stored patterns}, we have that
    $$
    \mathbb{P}(A) \geq 1-p .
    $$ 
    Now , under the event A, the means $\mu_i$ are pairwise $d_{\min }$-separated, Lemma \ref{mean separation of patterns implies margin separation of Sinkhorn divergences} applies. Thus, for every $i$, every $\xi \in \ProbM$ with $S_{\varepsilon}\left(\xi, X_i\right) \leq r$, and every $j \neq i$,
    $$
    S_{\varepsilon}\left(\xi, X_j\right)-S_{\varepsilon}\left(\xi, X_i\right) \geq \Delta = \frac{d_{\min}^{2}}{4}.
    $$
    This proves the margin-separation statement. It remains to prove pairwise disjointness of the basins. Suppose for contradiction that for some $i \neq j$ there exists
    $$
    \xi \in B_i(r) \cap B_j(r) .
    $$
    Since $\xi \in B_i(r)$, the margin-separation statement with indices $(i, j)$ gives
    $$
    S_{\varepsilon}\left(\xi, X_j\right)-S_{\varepsilon}\left(\xi, X_i\right) \geq \Delta .
    $$
    Since $\xi \in B_j(r)$, the same statement with indices $(j, i)$ gives
    $$
    S_{\varepsilon}\left(\xi, X_i\right)-S_{\varepsilon}\left(\xi, X_j\right) \geq \Delta .
    $$
    Adding these two inequalities yields
    $$
    0 \geq 2 \Delta,
    $$
    which is impossible since $\Delta=\frac{d_{\min}^{2}}{4}>0$.
    Hence
    $$
    B_i(r) \cap B_j(r)=\emptyset \quad \text { for all } i \neq j .
    $$
    This completes the proof.

\end{proof}

\subsection{ Proof of Theorem \ref{Geometric convergence in Sinkhorn divergence to the local minimizer}}

\begin{proof}
    Assume first that $\xi^{(k)} \in B_i(r)$ for all  $k \geq 0$. By Lemma \ref{local basin compactness inside parameter space}, after relabeling the atoms of each iterate, there is a unique ordered representative
    $$
    z^{(k)}=\left(a^{(k)}, x^{(k)}\right) \in K_i\left(\delta_i, \tau_i\right)
    $$
    such that
    $$
    \xi^{(k)}=\Xi\left(z^{(k)}\right).
    $$
    Since $\eta<\etaretarg{i}$, Lemma \ref{uniform retraction domain and local metric upper bound} implies that the local retraction is well-defined at each iterate. 
    
    Let $\xi^{(k+1)}=\Phi_\eta\left(\xi^{(k)}\right)$. For each $k$, define the one-step retraction curve
    $$
    \gamma_k(t)\coloneqq\operatorname{Ret}_{\xi^{(k)}}\left(-t \eta \operatorname{grad}_{\operatorname{SHK}} E\left(\xi^{(k)}\right)\right), \quad t \in[0,1] .
    $$
    
    Then $\gamma_k(0)=\xi^{(k)}$ and $\gamma_k(1)=\xi^{(k+1)}$. Applying Lemma \ref{Bounding SHK distance in terms of SHK gradient norm along retraction curve} with $(\delta a, \delta x)=-\eta \operatorname{grad}_{\operatorname{SHK}} E\left(\xi^{(k)}\right)$ in particle coordinates yields

    $$
    \operatorname{Length}\left(\gamma_k\right) \leq e^{\|\delta a / a\|_{\infty}} \eta\left\|\operatorname{grad}_{\operatorname{SHK}} E\left(\xi^{(k)}\right)\right\|_{\operatorname{SHK},\xi^{(k)}} .
    $$
    Now we proceed to bound $\|\delta a / a\|_{\infty}$. For the SHK gradient descent direction, the weight update follows
    $$
    \delta a_m=-\eta \frac{a_m}{\lambda^2} u_{\xi^{(k)}}\left(x_m\right) \quad \Longrightarrow \quad \frac{\delta a_m}{a_m}=-\frac{\eta}{\lambda^2} u_{\xi^{(k)}}\left(x_m\right) 
    $$
    where \[u_{\xi} \coloneqq \frac{\delta E}{\delta\xi}(\xi) - \Big\langle \frac{\delta E}{\delta\xi}(\xi), \xi \Big\rangle = \frac{\delta E}{\delta\xi}(\xi) - \int \frac{\delta E}{\delta\xi}(\xi) d \xi.\]
    Thus
    $$
    \left\|\frac{\delta a}{a}\right\|_{\infty} \leq \frac{\eta}{\lambda^2} \sup _{x \in \Omega}\left|u_{\xi^{(k)}}(x)\right| .
    $$
    Now, $\frac{\delta E}{\delta\xi}(\xi) = \sum_{j=1}^{N}w_j(\xi) \frac{\delta S_{\varepsilon}(\xi,X_j)}{\delta\xi}$ and hence $u_{\xi}(x) = \sum_{j=1}^{N}w_j(\xi) \left(\phi_i(x)- \int \phi_i(y) d\xi(y)\right)$ where 
    $$
    \phi_i(x) \coloneqq \left(\frac{\delta S_{\varepsilon}(\xi,X_i)}{\delta\xi}\right)(x)
    $$
    From Lemma \ref{uniform retraction domain and local metric upper bound} Part (ii), we have that $\sup_{x \in \Omega} |u_{\xi}(x)| \leq D^2$. Consequently, we have that
    $$
    \left\|\frac{\delta a}{a}\right\|_{\infty} \leq \frac{\eta D^2}{\lambda^2}.
    $$
    Therefore, we have that
    $$
    \operatorname{Length}\left(\gamma_k\right) \leq e^{\eta D^2 / \lambda^2} \eta\left\|\operatorname{grad}_{\operatorname{SHK}} E\left(\xi^{(k)}\right)\right\|_{\operatorname{SHK},\xi^{(k)}} .
    $$
    Now, using Lemmas \ref{Bounding SHK distance in terms of SHK gradient norm along retraction curve} and \ref{Energy gap contraction}, we have that
    $$\begin{aligned}
    d_{\operatorname{SHK}}\left(\xi^{(k+1)}, \xi^{(k)}\right) \leq & \operatorname{Length}\left(\gamma_k\right) \leq \min\left\{e^{\eta D^2 / \lambda^2},\frac{1}{\sqrt{\amin}}\right\} \eta\left\|\operatorname{grad}_{\operatorname{SHK}} E\left(\xi^{(k)}\right)\right\|_{\operatorname{SHK},\xi^{(k)}}\\
    \leq & \min\left\{e^{\eta D^2 / \lambda^2},\frac{1}{\sqrt{\amin}}\right\} \eta \cdot \sqrt{\frac{2 \left(E\left(\xi^{(0)}\right)-E_i^{*}(r)\right)}{\eta}} (1-\eta \mu)^\frac{k}{2} 
    = C q^\frac{k}{2}
    \end{aligned}
    $$
    where $C \coloneqq  \min\left\{e^{\eta D^2 / \lambda^2},\frac{1}{\sqrt{\amin}}\right\} \cdot \sqrt{2\eta \left(E\left(\xi^{(0)}\right)-E_i^{*}(r)\right)}$ and $q\coloneqq1-\mu \eta \in(0,1)$.

    Now, for integers $m>k$,
    $$
d_{\operatorname{SHK}}\left(\xi^{(m)}, \xi^{(k)}\right) \leq \sum_{t=k}^{m-1} d_{\operatorname{SHK}}\left(\xi^{(t+1)}, \xi^{(t)}\right) \leq \sum_{t=k}^{m-1} C q^{t / 2} \leq C \sum_{t=k}^{\infty} q^{t / 2}=C \frac{q^{k / 2}}{1-\sqrt{q}} .
    $$
    As $k \rightarrow \infty$, the RHS converges to zero. Therefore, $\left(\xi^{(k)}\right)$ is Cauchy in $d_{\SHK}$.

    The set $K_i\left(\delta_i, \tau_i\right)$ is compact by Lemma \ref{local basin compactness inside parameter space}. Therefore the sequence $z^{(k)}$ has a Euclidean-convergent subsequence:
    $$
    z^{\left(k_j\right)} \rightarrow z^{\infty}=\left(a^{\infty}, x^{\infty}\right) \in K_i\left(\delta_i, \tau_i\right) .
    $$
    Define
    $$
    \xi^{\infty}\coloneqq\Xi\left(z^{\infty}\right) \in P_M(\Omega)
    $$
    Because $K_i\left(\delta_i, \tau_i\right) \subset \aspace \times \locspace(\Omega)$, the limit belongs to the parameter space $\aspace \times \locspace(\Omega)$.

    Using Lemma \ref{local basin compactness inside parameter space}, we have that
    $$
    d_{\SHK}\left(\xi^{\left(k_j\right)}, \xi^{\infty}\right) \leq L_i\left\|z^{\left(k_j\right)}-z^{\infty}\right\|_E \rightarrow 0 .
    $$
    Hence, the subsequence converges to $\xi^{\infty}$ in $d_{\SHK}$.

    Let $\delta>0$. Since $\left(\xi^{(k)}\right)$ is $d_{\SHK}$-Cauchy, there exists $N$ such that
    $d_{\SHK}\left(\xi^{(m)}, \xi^{(n)}\right)<\delta / 2 \quad \forall m, n \geq N$.
    Choose $j$ so large that $k_j \geq N$ and $d_{\SHK}\left(\xi^{\left(k_j\right)}, \xi^{\infty}\right)<\delta / 2.
    $
    Then, for every $k \geq N$,
    $$
    d_{\SHK}\left(\xi^{(k)}, \xi^{\infty}\right) \leq d_{\SHK}\left(\xi^{(k)}, \xi^{\left(k_j\right)}\right)+d_{\SHK}\left(\xi^{\left(k_j\right)}, \xi^{\infty}\right)<\delta .
    $$
    Therefore, $\xi^{(k)} $ converges to $\xi^{\infty}$ in $d_{\SHK}$.

    Using Lemma \ref{Lipschitz continuity of E with respect to SHK metric}, each $F_i = S_{\varepsilon}(\cdot,X_i)$ is $G$-Lipschitz in $d_{\SHK}$, and hence,
    $$
        \left|F_i\left(\xi^{(k)}\right)-F_i\left(\xi^{\infty}\right)\right| \leq G d_{\SHK}\left(\xi^{(k)}, \xi^{\infty}\right) \rightarrow 0.
    $$
    Since $F_i\left(\xi^{(k)}\right) \leq r$ for all $k$, $F_i\left(\xi^{\infty}\right) \leq r$. Hence, we have that
    $\xi^{\infty} \in B_i(r)$.

    Again, using Lemma  \ref{Lipschitz continuity of E with respect to SHK metric}, we have that
    $$
        \left|E\left(\xi^{(k)}\right)-E\left(\xi^{\infty}\right)\right| \leq G d_{\SHK}\left(\xi^{(k)}, \xi^{\infty}\right) \rightarrow 0.
    $$

    On the other hand, Lemma \ref{Energy gap contraction} gives $E\left(\xi^{(k)}\right) \rightarrow E_i^*(r)$. Therefore, we have that $E\left(\xi^{\infty}\right)=E_i^*(r)$.
    Therefore, $\xi^{\infty}$ is a minimizer of $E$ on $B_i(r)$. Under Assumption \ref{Ass: Existence of minimizer of E in local basin}, by uniqueness, we have that
    $$
    \xi^{\infty}=X_i^*(r)
    $$
    This proves the convergence of the SHK gradient descent iterates  in $d_{\SHK}$ to $X_i*(r)$.

    Now, since $d_{\operatorname{SHK}}\left(\xi^{(m)}, \xi^{(k)}\right) \leq  C \frac{q^{k / 2}}{1-\sqrt{q}} $ and $d_{\operatorname{SHK}}$ is continuous, we have that
    $$
        d_{\operatorname{SHK}}\left(X_i^{*}(r), \xi^{(k)}\right) = d_{\operatorname{SHK}}\left(\xi^{(\infty)}, \xi^{(k)}\right) = \lim_{m \to \infty} d_{\operatorname{SHK}}\left(\xi^{(m)}, \xi^{(k)}\right)  \leq  C \frac{q^{k / 2}}{1-\sqrt{q}}. 
    $$

    Finally, using Lemma \ref{Open bounded convex domain implies boundedness of SHK gradient of Sinkhorn divergence} and the fact that $S_{\varepsilon}\left(X_i^*(r), X_i^*(r)\right)=0$, we have that
    $$
    S_{\varepsilon}\left(\xi^{(k)}, X_i^*(r)\right) \leq G d_{\operatorname{SHK}}\left(X_i^{*}(r), \xi^{(k)}\right) \leq  GC \frac{q^{k / 2}}{1-\sqrt{q}}.
    $$

    Consequently, $\lim_{k \to \infty} S_{\varepsilon}(\xi^{(k)}, X_i^{*}(r)) = 0$ and since the Sinkhorn divergence metrizes weak convergence (Theorem 1 of \cite{feydy2019interpolating}), we have that $\xi^{(k)}$ converges weakly to $X_i^{*}(r)$.

    Now, $S_{\varepsilon}\left(\xi^{(k)}, X_i^*(r)\right) \leq  G C \frac{q^{k / 2}}{1-\sqrt{q}} \leq \delta$ if 
    $$
    \begin{aligned}
    k \geq & \frac{2\log \left(\frac{ \min\left\{e^{\eta D^2 / \lambda^2},\frac{1}{\sqrt{\amin}}\right\} \cdot G\sqrt{2\eta \left(E\left(\xi^{(0)}\right)-E_i^{*}(r)\right)}}{\delta(1- \sqrt{1- \mu\eta})}\right)}{- \log (1-\mu \eta)}\\
    =& \frac{\min\left\{\frac{2\eta D^2}{\lambda^2}, - \log \amin\right\} + 2\log G+\log(2\eta \left(E\left(\xi^{(0)}\right)-E_i^{*}(r)\right)) + 2\log (\frac{1}{\delta(1- \sqrt{1- \mu\eta})})}{- \log (1-\mu \eta)}.
    \end{aligned}
    $$ Using $1 + \log x \leq x$ for $0 < x< 1$, 
    we have that $\frac{1}{-\log (1-\mu \eta)} \leq \frac{1}{\mu \eta}$. Further, we have that $1-\sqrt{1-\mu \eta}=\frac{\mu \eta}{1+\sqrt{1-\mu \eta}} \geq \frac{\mu \eta}{2} \implies \frac{1}{1-\sqrt{1-\mu \eta}} \leq \frac{2}{\mu \eta}$. Consequently,
    we have the sufficient condition $k \geq \min\left\{\frac{2D^2}{\mu\lambda^2}, - \frac{\log \amin}{\mu\eta}\right\} + \frac{1}{\mu \eta} \log(2\eta \left(E\left(\xi^{(0)}\right)-E_i^{*}(r)\right)) + \frac{2}{\mu \eta}\log (\frac{2G}{\delta \mu \eta})$.

    We will now establish that if $$
F_i\left(\xi^{(0)}\right) \leq r-\rho_i(\eta,r,\xi^{(0)}),
$$

then $\xi^{(k)} \in B_i(r)$ for every $k \in \mathbb{N}$. We will prove by strong induction that $F_i\left(\xi^{(k)}\right) \leq r$ for all $k$.

Note that, by assumption $F_i\left(\xi^{(0)}\right) \leq r-\rho_i(\eta,r,\xi^{(0)})<r$.

Now, assume that $F_i\left(\xi^{(t)}\right) \leq r$ for all $t=0,1, \ldots, k$.
Then $\xi^{(t)} \in B_i(r)$ for those $t$, so Assumptions \ref{Ass: L-smoothness of energy functional E in SHK geometry} and \ref{Ass: PL inequality in local basin} apply on each of these iterates, and consequently Lemma \ref{Energy gap contraction} applies up to time $k$, yielding

$$
E\left(\xi^{(t)}\right)-E_i^{\star}(r) \leq \left(1-\mu \eta\right)^t \left(E\left(\xi^{(0)}\right)-E_i^{*}(r)\right), \quad t=0, \ldots, k
$$

We now bound the increment of $F_i$ from $\xi^{(t)}$ to $\xi^{(t+1)}$. To do so, we construct an explicit smooth curve in $\mathcal{M}_M$ connecting them and bound its SHK length.

Along any absolutely continuous curve $\zeta(s)$ connecting $\zeta(0)=\xi^{(t)}$ to $\zeta(1) = \xi^{(t+1)}$ restricted to $\ProbM$,

$$
\frac{d}{d s} F_i(\zeta(s))=\left\langle\operatorname{grad}_{\operatorname{SHK}} F_i(\zeta(s)), \frac{d}{ds}\zeta(s)\right\rangle_{\operatorname{SHK},\zeta(s)} \leq\left\|\operatorname{grad}_{\operatorname{SHK}} F_i(\zeta(s))\right\|_{\operatorname{SHK},\zeta(s)} \cdot\left\|\frac{d}{ds}\zeta(s)\right\|_{\operatorname{SHK},\zeta(s)} .
$$

Integrating between 0 and 1 and using the global bound $\left\|\operatorname{grad}_{\operatorname{SHK}} F_i(\xi)\right\|_{\operatorname{SHK},\xi} \leq G$ from Lemma \ref{Open bounded convex domain implies boundedness of SHK gradient of Sinkhorn divergence}, we obtain
$$
F_i\left(\xi^{(t+1)}\right)-F_i\left(\xi^{(t)}\right) \leq G \cdot \operatorname{Length}(\zeta) .
$$
Using Lemmas \ref{Energy gap contraction} and \ref{Bounding SHK distance in terms of SHK gradient norm along retraction curve}, we have that
$$
\begin{aligned}
F_i\left(\xi^{(t+1)}\right)-F_i\left(\xi^{(t)}\right) \leq & G \cdot \frac{\eta}{\sqrt{\amin}} \cdot \sqrt{\frac{2 \left(E\left(\xi^{(0)}\right)-E_i^{*}(r)\right)}{\eta}} (1-\mu \eta)^\frac{t}{2}\\
=&G \cdot \sqrt{\frac{2 \eta}{\amin}} \sqrt{\left(E\left(\xi^{(0)}\right)-E_i^{*}(r)\right)} (1-\mu \eta)^\frac{t}{2} .
\end{aligned}
$$
Now summing from $t=0$ to $t=k$ :
$$
F_i\left(\xi^{(k+1)}\right) \leq F_i\left(\xi^{(0)}\right)+G \sqrt{\frac{2 \eta}{\amin}} \sqrt{\left(E\left(\xi^{(0)}\right)-E_i^{*}(r)\right)} \sum_{t=0}^k (1-\mu \eta)^\frac{t}{2} .
$$

Since $\sum_{t=0}^k (1-\mu \eta)^\frac{t}{2} \leq \sum_{t=0}^{\infty}(1-\mu \eta)^\frac{t}{2}=\frac{1}{1-\sqrt{1-\mu \eta}}$,

$$
F_i\left(\xi^{(k+1)}\right) \leq F_i\left(\xi^{(0)}\right)+G\frac{\sqrt{2\eta\left(E\left(\xi^{(0)}\right)-E_i^{*}(r)\right)}}{\sqrt{\amin}(1-\sqrt{1-\mu \eta})}=F_i\left(\xi^{(0)}\right)+\rho_i(\eta,r,\xi^{(0)}) .
$$

By the assumed initialization condition $F_i\left(\xi^{(0)}\right) \leq r-\rho_i(\eta,r,\xi^{(0)})$, we get $F_i\left(\xi^{(k+1)}\right) \leq r$. This completes the induction.
Thus all iterates remain in $B_i(r)$.

The lower bound on $\eta$ can be derived based on the natural constraint that $\rho_i(\eta,r,\xi^{(0)}) \leq r$ must be satisfied. The, we must have that
$$
G\frac{\sqrt{2\eta\left(E\left(\xi^{(0)}\right)-E_i^{*}(r)\right)}}{\sqrt{\amin}(1-\sqrt{1-\mu \eta})} \leq r \iff \frac{2 G^2 \eta \left(E\left(\xi^{(0)}\right)-E_i^{*}(r)\right)}{\amin(1-\sqrt{1-\mu \eta})^2} \leq r^2.
$$

Define $u\coloneqq \sqrt{1-\mu \eta}$. Since $0<\eta<\frac{1}{\mu}$, we have that $0<u<1$. Further,
$\eta=\frac{1-u^2}{\mu}=\frac{(1-u)(1+u)}{\mu}$. Therefore, the condition reduces to
$$
\frac{1+u}{1-u} \leq r^2 \times\frac{\mu \amin}{2 G^2 \left(E\left(\xi^{(0)}\right)-E_i^{*}(r)\right)} \eqqcolon \alpha(r,\xi^{(0)}).
$$
Because $0<u<1$, we have that $\frac{1+u}{1-u}>1$. Therefore, a necessary condition is $\alpha(r,\xi^{(0)})>1$. 
Moreover, if $\alpha(r,\xi^{(0)})>1$, then the condition is equivalent to
$$
1+u \leq \alpha(r,\xi^{(0)})(1-u) \iff u \leq \frac{\alpha(r,\xi^{(0)}) - 1}{\alpha(r,\xi^{(0)})+1} \iff 1-\mu \eta \leq\left(\frac{\alpha(r,\xi^{(0)})-1}{\alpha(r,\xi^{(0)})+1}\right)^2.
$$ 

Consequently, the condition reduces to $\alpha(r,\xi^{(0)})>1$ and $\eta \geq \frac{1}{\mu} \times \left[1-\left(\frac{\alpha(r,\xi^{(0)})-1}{\alpha(r,\xi^{(0)})+1}\right)^2\right] = \frac{4 \alpha(r,\xi^{(0)})}{\mu(\alpha(r,\xi^{(0)})+1)^2} .$

    This completes the proof.

\end{proof}

\subsubsection{Auxiliary results for proving Theorem \ref{Geometric convergence in Sinkhorn divergence to the local minimizer}}

For any $i=1,\dots,N$ consider the stored pattern $X_i=\sum_{m=1}^M b_{i, m} \delta_{y_{i, m}} \in \ProbM$ whose weight and locations parameters are denoted as $b_i\coloneqq\left(b_{i, 1}, \dots, b_{i, M}\right)$ and $y_i\coloneqq\left(y_{i, 1}, \ldots, y_{i, M}\right)$, respectively. Let us define the stored pattern margins
$$
\begin{gathered}
w_i\coloneqq\min _{1 \leq m \leq M}\left(b_{i, m}-\amin\right)>0, \\
d_i^{\partial}\coloneqq\min _{1 \leq m \leq M} \operatorname{dist}\left(y_{i, m}, \partial \Omega\right)>0, \\
s_i \coloneqq \operatorname{sep}\left(y_i\right)=\min _{m \neq n}\left\|y_{i, m}-y_{i, n}\right\|_2>\Dsep.
\end{gathered}
$$
Further, define $\bar{\delta}_i\coloneqq\min \left\{d_i^{\partial}, \frac{s_i-\Dsep}{2}\right\}>0$. Then, for any $0 < \delta < \bar{\delta}_i$ and $0 < \tau < w_i$, let us define
\begin{equation}\label{Definition of local radius boudn determined by stored pattern}
    \rloc_i(\delta, \tau)\coloneqq\min \left\{\frac{\amin \delta^2}{2}-\varepsilon \log M, \frac{\tau\left(s_i-\delta\right)^2}{4}-\varepsilon \log M\right\} .
\end{equation}

\begin{lemma}[local basin compactness inside parameter space]\label{local basin compactness inside parameter space}
    Fix $i \in\{1, \dots, N\}$, and choose numbers $\delta_i, \tau_i$ such that $0<\delta_i<\bar{\delta}_i, \quad 0<\tau_i<w_i$.
    Assume $0<r<r_i^{\mathrm{loc}}\left(\delta_i, \tau_i\right)$.
    Then every $\xi=\sum_{m=1}^M a_m \delta_{x_m} \in B_i(r)$ admits, after relabeling of its atoms, a unique ordered representative $(a,x)$ satisfying
    $$
    \left\|a-b_i\right\|_1 \leq \tau_i, \quad x_m \in \bar{\mathbb{B}}\left(y_{i, m}, \delta_i\right) \quad \text { for every } m=1, \ldots, M .
    $$
    Consequently,
    $$
    B_i(r) \subset \Xi\left(K_i\left(\delta_i, \tau_i\right)\right),
    $$
    where
    $$
    K_i\left(\delta_i, \tau_i\right)\coloneqq\left\{(a, x): \sum_{m=1}^M a_m=1,\left\|a-b_i\right\|_1 \leq \tau_i, x_m \in \bar{\mathbb{B}}\left(y_{i, m}, \delta_i\right) \forall m\right\},
    $$ 
    and $K_i\left(\delta_i, \tau_i\right)$ is a compact subset of the parameter space $\mathcal{M}_M = \aspace \times \locspace(\Omega)$.
        \end{lemma}

    \begin{proof}
        From Lemma \ref{mean separation of patterns implies margin separation of Sinkhorn divergences}, we have that
        $ \operatorname{OT}_{\varepsilon}(\xi, \xi) \leq \varepsilon \log M$ and $\operatorname{OT}_{\varepsilon}\left(X_i, X_i\right) \leq \varepsilon \log M$.

        Therefore,
        \begin{equation}\label{lower bound on Sinkhorn divergence in local radius}
        S_{\varepsilon}\left(\xi, X_i\right) \geq \operatorname{OT}_{\varepsilon}\left(\xi, X_i\right)-\varepsilon \log M .
        \end{equation}
        Now suppose $\xi \in B_i(r)$. Assume for the sake of contradiction that some query atom $x_{m_0}$ does not belong to $\bigcup_{m=1}^M \bar{\mathbb{B}}\left(y_{i, m}, \delta_i\right)$. Then, for every target atom $y_{i, m}$,
        $$
        \left\|x_{m_0}-y_{i, m}\right\|_2>\delta_i
        $$
        and hence every unit of mass transported out of row $m_0$ must pay least $\delta_i^2 / 2$ in transport cost. Since the row mass equals $a_{m_0} \geq \amin$, we must have that
        $$
        \operatorname{OT}_{\varepsilon}\left(\xi, X_i\right) \geq \frac{\amin \delta_i^2}{2}.
        $$
        Using ~\eqref{lower bound on Sinkhorn divergence in local radius},
        $$
        S_{\varepsilon}\left(\xi, X_i\right) \geq \frac{\amin \delta_i^2}{2}-\varepsilon \log M.
        $$
        But $r<\rloc_i\left(\delta_i, \tau_i\right) \leq \frac{\amin \delta_i^2}{2}-\varepsilon \log M$, contradicting $S_{\varepsilon}\left(\xi, X_i\right) \leq r$.
        Hence, each query atom lies in $\bigcup_m \bar{\mathbb{B}}\left(y_{i, m}, \delta_i\right)$.

        Now, assume for the sake of contradiction that for some $n_0$, no query atom belongs to $\bar{\mathbb{B}}\left(y_{i, n_0}, \delta_i\right)$. Then every unit of mass transported into column $n_0$ pays at least $\delta_i^2 / 2$. Since the column mass equals $b_{i, n_0}>a_{\text {min }}$,
        $$
        \operatorname{OT}_{\varepsilon}\left(\xi, X_i\right) \geq \frac{b_{i, n_0} \delta_i^2}{2} \geq \frac{\amin \delta_i^2}{2}.
        $$
        Again ~\eqref{lower bound on Sinkhorn divergence in local radius} gives
        $$
        S_{\varepsilon}\left(\xi, X_i\right) \geq \frac{\amin \delta_i^2}{2}-\varepsilon \log M>r,
        $$
        which is a contradiction. So every closed ball $\bar{\mathbb{B}}\left(y_{i, n}, \delta_i\right)$ contains at least one query atom.

        Since $\delta_i<\bar{\delta}_i \leq \frac{s_i-\Dsep}{2}<\frac{s_i}{2}$, the balls $\bar{\mathbb{B}}\left(y_{i, n}, \delta_i\right)$ are pairwise disjoint. Based on our arguments above, all $M$ query atoms lie in the union of these $M$ disjoint balls, and each ball contains at least one query atom. Since there are exactly $M$ query atoms, each ball contains exactly one. Therefore, after a unique relabeling, we may assume
        \begin{equation}\label{uniqueness of atom labellings}
        x_m \in \bar{\mathbb{B}}\left(y_{i, m}, \delta_i\right), \quad m=1, \ldots, M .
        \end{equation}

        Let $P=\left(P_{m n}\right)_{m=1, n=1}^{M,M}$ be any coupling matrix between $\xi$ and $X_i$, i.e.
        $$
        P \geq 0, \quad P \mathbf{1}=a, \quad P^{\top} \mathbf{1}=b_i .
        $$
        Then, we have that $\sum_{m=1}^M P_{m m} \leq \sum_{m=1}^M \min \left(a_m, b_{i, m}\right)=1-\frac{1}{2}\left\|a-b_i\right\|_1$ and the off-diagonal mass satisfies 
        \begin{equation}\label{off diagonal}
        \sum_{m \neq n} P_{m n}=1-\sum_{m=1}^M P_{m m} \geq \frac{1}{2}\left\|a-b_i\right\|_1 . 
        \end{equation}

        Now, fix some $m \neq n$. Then, using ~\eqref{uniqueness of atom labellings}, we have that, $\left\|x_m-y_{i, n}\right\|_2 \geq\left\|y_{i, m}-y_{i, n}\right\|_2-\left\|x_m-y_{i, m}\right\|_2 \geq s_i-\delta_i$. Therefore
        $$
        c\left(x_m, y_{i, n}\right)=\frac{1}{2}\left\|x_m-y_{i, n}\right\|_2^2 \geq \frac{1}{2}\left(s_i-\delta_i\right)^2 .
        $$
        So every unit of off-diagonal mass contributes a transport cost of at least $\frac{1}{2}\left(s_i-\delta_i\right)^2$. Using ~\eqref{uniqueness of atom labellings}, we have that
        $$
        \operatorname{OT}_{\varepsilon}\left(\xi, X_i\right) \geq \frac{1}{2}\left(s_i-\delta_i\right)^2 \sum_{m \neq n} P_{m n} \geq \frac{\left(s_i-\delta_i\right)^2}{4}\left\|a-b_i\right\|_1 .
        $$
        Combining with ~\eqref{lower bound on Sinkhorn divergence in local radius}, we have that
        $$
        S_{\varepsilon}\left(\xi, X_i\right) \geq \frac{\left(s_i-\delta_i\right)^2}{4}\left\|a-b_i\right\|_1-\varepsilon \log M .
        $$
        If $\left\|a-b_i\right\|_1>\tau_i$, then
        $$
        S_{\varepsilon}\left(\xi, X_i\right)>\frac{\left(s_i-\delta_i\right)^2}{4} \tau_i-\varepsilon \log M \geq r_i^{\mathrm{loc}}\left(\delta_i, \tau_i\right)>r
        $$
        which leads to a contradiction. Thus, we must have that
        $$
        \left\|a-b_i\right\|_1 \leq \tau_i .
        $$

        The defining constraints of $K_i\left(\delta_i, \tau_i\right)$ are closed and bounded in the finite-dimensional Euclidean space $\mathbb{R}^M \times\left(\mathbb{R}^d\right)^M$, so $K_i\left(\delta_i, \tau_i\right)$ is compact. We only need to verify that $K_i\left(\delta_i, \tau_i\right)$ is indeed a subset of the parameter space $\aspace \times \locspace(\Omega)$. If $(a, x) \in K_i\left(\delta_i, \tau_i\right)$, then, for all $m =1,\dots,M$,
        $$
        a_m \geq b_{i, m}-\tau_i>\amin,
        $$
        since $\tau_i<w_i=\min _m\left(b_{i, m}-\amin\right)$.
        Also, since $\delta_i<d_i^{\partial}$, we have that, for every $m=1,\dots,M$
        $$
        \bar{\mathbb{B}}\left(y_{i, m}, \delta_i\right) \subset \Omega \quad \text { for every } m .
        $$
        Finally, for $m \neq n$, we have that
        $$
        \left\|x_m-x_n\right\| \geq\left\|y_{i, m}-y_{i, n}\right\|-2 \delta_i \geq s_i-2 \delta_i>\Dsep
        $$
        since $\delta_i<\frac{s_i-\Dsep}{2}$. Therefore, we finally have that
        $$
        K_i\left(\delta_i, \tau_i\right) \subset \aspace \times \locspace(\Omega) .
        $$
        This completes the proof.
    \end{proof}

    \begin{lemma}[uniform retraction domain and local metric upper bound]\label{uniform retraction domain and local metric upper bound}
        Assume $0<r<r_i^{\mathrm{loc}}\left(\delta_i, \tau_i\right)$, and let $K_i\left(\delta_i, \tau_i\right)$ be as in Lemma \ref{local basin compactness inside parameter space}. Define
        $$
        a_i^{-}\coloneqq\min _{1 \leq m \leq M} b_{i, m}-\tau_i>\amin
        $$
        and
        $$
        \etaretarg{i}\coloneqq\min \left\{\frac{\lambda^2}{2 D^2} \log \frac{a_i^{-}}{\amin}, \frac{1}{2 D} \min \left\{d_i^{\partial}-\delta_i, s_i-2 \delta_i-\Dsep\right\}\right\} .
        $$
        Then, the following hold true
        \begin{enumerate}
            \item For every $\xi=\Xi(a, x) \in \Xi\left(K_i\left(\delta_i, \tau_i\right)\right)$,
            $$
            \sup _{z \in \Omega}\left|u_{\xi}(z)\right| \leq D^2, \quad \sup _{z \in \Omega}\left\|\nabla u_{\xi}(z)\right\| \leq D .
            $$
            \item If $0<\eta<\etaretarg{i}$, then for every $\xi=\Xi(a, x) \in \Xi\left(K_i\left(\delta_i, \tau_i\right)\right)$, the retraction
            $\operatorname{Ret}_{\xi}\left(-\eta \operatorname{grad}_{\operatorname{SHK}} E(\xi)\right)$ is well-defined.
            \item If
            $$
            \|(\delta a, \delta x)\|_E^2\coloneqq\sum_{m=1}^M\left(\delta a_m\right)^2+\sum_{m=1}^M\left\|\delta x_m\right\|^2,
            $$
            then on $K_i\left(\delta_i, \tau_i\right)$,
            $$
            \|(\delta a, \delta x)\|_{\mathrm{SHK},(a, x)}^2 \leq L_i^2\|(\delta a, \delta x)\|_E^2, \quad L_i\coloneqq\sqrt{\max \left\{\frac{\lambda^2}{a_i^{-}}, 1\right\}} .
            $$
            Consequently, for all $(a, x),\left(a^{\prime}, x^{\prime}\right) \in K_i\left(\delta_i, \tau_i\right)$,
            $$
            d_{\operatorname{SHK}}\left(\Xi(a, x), \Xi\left(a^{\prime}, x^{\prime}\right)\right) \leq L_i\left\|\left(a-a^{\prime}, x-x^{\prime}\right)\right\|_E .
            $$
            \end{enumerate}
    \end{lemma}

    \begin{proof}
        For $j=1, \ldots, N$, let
        $$
        \phi_{j, \xi}(z)\coloneqq\frac{\delta S_{\varepsilon}\left(\xi, X_j\right)}{\delta \xi}(z)
        $$
        and let
        $$
        \bar{\phi}_{j, \xi}(z)\coloneqq\phi_{j, \xi}(z)-\int \phi_{j, \xi} d \xi .
        $$
        Then
        $$
        u_{\xi}(z)=\sum_{j=1}^N w_j(\xi) \bar{\phi}_{j, \xi}(z), \quad \sum_{j=1}^N w_j(\xi)=1, \quad w_j(\xi) \geq 0 .
        $$
        Following the derivations in the proof of Lemma \ref{Open bounded convex domain implies boundedness of SHK gradient of Sinkhorn divergence}, specifically~\eqref{Bound on norm of gradient of first variation} and~\eqref{Bound on norm of oscillation of first variation}, for each $j=1,\dots,N$, we have that
        $$
        \sup _{z \in \Omega}\left\|\nabla \phi_{j, \xi}(z)\right\| \leq D, \quad \sup _{z \in \Omega} \phi_{j, \xi}(z)-\inf _{z \in \Omega} \phi_{j, \xi}(z) \leq D^2 .
        $$
        Hence, we obtain
        $$
        \begin{aligned}
        &\sup _{z \in \Omega}\left|\bar{\phi}_{j, \xi}(z)\right| \leq D^2 \,\ \textrm{ and }\,\ \sup _{z \in \Omega}\left\|\nabla \bar{\phi}_{j, \xi}(z)\right\| \leq D .
        \end{aligned}
        $$
        Taking the convex combination with weights $w_j(\xi)$, we have that, 
        $$
        \begin{aligned}
        &\sup _{z \in \Omega}\left|u_{\xi}(z)\right| \leq D^2 \,\ \textrm{ and }\,\ \sup _{z \in \Omega}\left\|\nabla u_{\xi}(z)\right\| \leq D .
        \end{aligned}
        $$
        Let
        $$
        \operatorname{grad}_{\SHK} E(a, x)=\left(\left(\frac{a_m}{\lambda^2} u_{\xi}\left(x_m\right)\right)_{m=1}^M,\left(\nabla u_{\xi}\left(x_m\right)\right)_{m=1}^M\right) .
        $$
        The actual retraction step is therefore
        $$
        \delta a_m=-\eta \frac{a_m}{\lambda^2} u_{\xi}\left(x_m\right), \quad \delta x_m=-\eta \nabla u_{\xi}\left(x_m\right) .
        $$
        Using the above bound,
        \begin{equation}
            \left\|\frac{\delta a}{a}\right\|_{\infty}=\max _m \frac{\left|\delta a_m\right|}{a_m} \leq \eta \frac{D^2}{\lambda^2}, \quad\|\delta x\|_{\infty, 2}\coloneqq\max _m\left\|\delta x_m\right\| \leq \eta D .
        \end{equation}
        By Lemma \ref{local basin compactness inside parameter space}, every $(a, x) \in K_i\left(\delta_i, \tau_i\right)$ satisfies $a_m \geq a_i^{-}$  for all $m=1,\dots,M$, $d_{\partial}(x) \geq d_i^{\partial}-\delta_i$ and $\operatorname{sep}(x)-\Delta_{\min } \geq s_i-2 \delta_i-\Dsep$.
        
        Hence, under the local region of validity of the retraction map, as introduced in Section \ref{sec: defining retraction map},
        $$
        \left\|\frac{\delta a}{a}\right\|_{\infty}<\frac{1}{2} \log \frac{a_i^{-}}{\amin}
        $$
        and
        $$
        \|\delta x\|_{\infty, 2}<\frac{1}{2} \min \left\{d_i^{\partial}-\delta_i, s_i-2 \delta_i-\Dsep\right\}
        $$
        are sufficient to make the weight and position retractions well-defined. This is guaranteed by $0<\eta<\etaretarg{i}$.
        
        On $K_i\left(\delta_i, \tau_i\right)$, we have $a_m \geq a_i^{-}$and $a_m \leq 1$. Therefore,
        $$
        \|(\delta a, \delta x)\|_{\mathrm{SHK},(a, x)}^2=\sum_{m=1}^M \frac{\lambda^2}{a_m}\left(\delta a_m\right)^2+\sum_{m=1}^M a_m\left\|\delta x_m\right\|^2 \leq \frac{\lambda^2}{a_i^{-}} \sum_m\left(\delta a_m\right)^2+\sum_m\left\|\delta x_m\right\|^2 \leq L_i^2\|(\delta a, \delta x)\|_E^2 .
        $$
        Now $K_i\left(\delta_i, \tau_i\right)$ is convex, since the weight constraints define a convex set, and each position constraint $x_m \in \bar{B}\left(y_{i, m}, \delta_i\right)$ defines a convex set, and intersection of convex sets is a convex set again. Hence the straight segment joining two points of $K_i\left(\delta_i, \tau_i\right)$ stays in $K_i\left(\delta_i, \tau_i\right)$. Integrating the above pointwise metric upper bound along that straight segment yields
        $$
        d_{\SHK}\left(\Xi(a, x), \Xi\left(a^{\prime}, x^{\prime}\right)\right) \leq L_i\left\|\left(a-a^{\prime}, x-x^{\prime}\right)\right\|_E .
        $$
        This completes the proof.
    \end{proof}

\subsection{Proof of Theorem \ref{Sinkhorn distance between minimizer in local basin and stored pattern and stability of stored pattern}}

\begin{proof}
    Since $F_i\left(X_i\right)=S_{\varepsilon}\left(X_i, X_i\right)=0 \leq r$, we have $X_i \in B_i(r)$. Assumption \ref{Ass: Existence of minimizer of E in local basin} guarantees the existence and uniqueness of the minimizer $X_i^{*}(r) \in B_i(r)$. Now, we apply Lemma \ref{Control on basin interference and softmin perturbation bounds of energy E inside local basin using separation margin} to $\xi=X_i^{*}(r)$, giving
    $$
    0 \leq F_i\left(X_i^{*}(r)\right)-E\left(X_i^{*}(r)\right) \leq \frac{1}{\beta} \log \left(1+(N-1) e^{-\beta \Delta}\right)
    $$
    Thus
    $$
    F_i\left(X_i^{*}(r)\right) \leq E\left(X_i^{*}(r)\right)+\frac{1}{\beta} \log \left(1+(N-1) e^{-\beta \Delta}\right) .
    $$
    But $X_i^{*}(r)$ minimizes $E$ over $B_i(r)$, and $X_i \in B_i(r)$, hence
    $$
    E\left(X_i^{*}(r)\right) \leq E\left(X_i\right) .
    $$
    Also $E(\xi) \leq F_i(\xi)$ for every $\xi \in B_i(r)$, so
    $$
    E\left(X_i\right) \leq F_i\left(X_i\right)=0 .
    $$
    Combining, we have that
    $$
    \begin{aligned}
    S_{\varepsilon}(X_i^{*}(r),X_i)  =& F_i\left(X_i^{*}(r)\right)\\
    \leq & E\left(X_i^{*}(r)\right)+\frac{1}{\beta} \log \left(1+(N-1) e^{-\beta \Delta}\right)\\
    \leq & E\left(X_i\right)+\frac{1}{\beta} \log \left(1+(N-1) e^{-\beta \Delta}\right)\\
    \leq & 0+\frac{1}{\beta} \log \left(1+(N-1) e^{-\beta \Delta}\right)\\
    =&\frac{1}{\beta} \log \left(1+(N-1) e^{-\beta \Delta}\right).
    \end{aligned}
    $$
    Now, consider the one-step retraction curve
    $$
    \gamma_i(t)\coloneqq\operatorname{Ret}_{X_i}\left(-t \eta \operatorname{grad}_{\operatorname{SHK}} E\left(X_i\right)\right), \quad t \in[0,1] .
    $$
    Then $\gamma_i(0)=X_i$ and $\gamma_i(1)=\Phi_\eta(X_i)$. Note that, as derived in Theorem \ref{Geometric convergence in Sinkhorn divergence to the local minimizer}, when $\eta \leq \etaretarg{i}$, the retraction $\Phi_\eta(X_i)$ is well defined. Applying Lemma \ref{Bounding SHK distance in terms of SHK gradient norm along retraction curve} with $(\delta a, \delta x)=-\eta \operatorname{grad}_{\operatorname{SHK}} E\left(X_i\right)$ in particle coordinates yields
    $$
    \operatorname{Length}\left(\gamma_i\right) \leq \min\left\{e^{\|\frac{\delta a}{a}\|_{\infty}},\frac{1}{\sqrt{\amin}}\right\} \eta\left\|\operatorname{grad}_{\operatorname{SHK}} E\left(X_i\right)\right\|_{\operatorname{SHK},X_i} .
    $$
    Now we proceed to bound $\|\frac{\delta a}{a}\|_{\infty}$. For the SHK gradient descent direction, the weight update follows
    $$
    \delta a_m=-\eta \frac{a_m}{\lambda^2} u_{X_i}\left(x_m\right) \quad \Longrightarrow \quad \frac{\delta a_m}{a_m}=-\frac{\eta}{\lambda^2} u_{X_i}\left(x_m\right) .
    $$
    where \[u_{\xi} \coloneqq \frac{\delta E}{\delta\xi}(\xi) - \Big\langle \frac{\delta E}{\delta\xi}(\xi), \xi \Big\rangle = \frac{\delta E}{\delta\xi}(\xi) - \int \frac{\delta E}{\delta\xi}(\xi) d \xi.\]
    Thus
    $$
    \left\|\frac{\delta a}{a}\right\|_{\infty} \leq \frac{\eta}{\lambda^2} \sup _{x \in \Omega}\left|u_{X_i}(x)\right| .
    $$
    Now, $\frac{\delta E}{\delta\xi}(\xi) = \sum_{j=1}^{N}w_j(\xi) \frac{\delta S_{\varepsilon}(\xi,X_j)}{\delta\xi}$ and hence $u_{\xi}(x) = \sum_{j=1}^{N}w_j(\xi) \left(\phi_i(x)- \int \phi_i(y) d\xi(y)\right)$ where 
    $$
    \phi_i(x) \coloneqq \left(\frac{\delta S_{\varepsilon}(\xi,X_i)}{\delta\xi}\right)(x)
    $$
    Following the argument presented in the proof of Lemma \ref{Open bounded convex domain implies boundedness of SHK gradient of Sinkhorn divergence}, specifically~\eqref{Bound on norm of oscillation of first variation}, we have that, for any $x,y \in \Omega$,  $ |\phi_i(x) - \phi_i(y)| \leq D^2$ and hence
    for any $\xi \in \ProbM$, we have that $\sup_{x \in \Omega} |u_{\xi}(x)| \leq D^2$. Consequently, we have that
    $$
    \left\|\frac{\delta a}{a}\right\|_{\infty} \leq \frac{\eta D^2}{\lambda^2}.
    $$
    Using Lemma \ref{Open bounded convex domain implies boundedness of SHK gradient of Sinkhorn divergence}, the fact that $S_{\varepsilon}(X_i,X_i) = 0$ and following the same argument as in the proof of Theorem \ref{Geometric convergence in Sinkhorn divergence to the local minimizer}, we have that $$
    \begin{aligned}
   S_{\varepsilon}(X_i,\Phi_\eta(X_i)) \leq & G d_{\operatorname{SHK}}(X_i, \Phi_\eta(X_i)) \\
   \leq & G\operatorname{Length}\left(\gamma_i\right) \\
   \leq &  \min\left\{e^{\eta D^2 / \lambda^2},\frac{1}{\sqrt{\amin}}\right\} \eta G\left\|\operatorname{grad}_{\operatorname{SHK}} E\left(X_i\right)\right\|_{\operatorname{SHK},X_i} .
   \end{aligned}
    $$
    Now, since $F_i(\xi) = S_{\varepsilon}(\xi,X_i)$ has a global minimum at $\xi = X_i$ and $F_i$ is differentiable in the SHK sense, we must have that $\operatorname{grad}_{\operatorname{SHK}} F_i(X_i) = 0$. Now, using Lemma \ref{Expression of SHK gradient oof energy functional E in terms of SHK gradients of Sinkhorn divergence}, we have that
    $$
    \operatorname{grad}_{\operatorname{SHK}} E(X_i) = \sum_{j=1}^{N}w_j(X_i) \operatorname{grad}_{\operatorname{SHK}} F_j(X_i)= \sum_{j\neq i}^{N}w_j(X_i) \operatorname{grad}_{\operatorname{SHK}} F_j(X_i).
    $$
    Consequently, using the gradient bound from Lemma \ref{Open bounded convex domain implies boundedness of SHK gradient of Sinkhorn divergence}, we have that
    $$
    \begin{aligned}
        \left\|\operatorname{grad}_{\operatorname{SHK}} E(X_i)\right\|_{\operatorname{SHK},X_i} \leq & \sum_{j \neq i} w_j (X_i)\left\|\operatorname{grad}_{\operatorname{SHK}} F_j(X_i)\right\|_{\operatorname{SHK},X_i} \\
        \leq & G \sum_{j \neq i} w_j (X_i) \leq \frac{G(N-1) e^{-\beta \Delta}}{1+(N-1) e^{-\beta \Delta}}
    \end{aligned}
    $$
    Therefore, we have that
    $$
    \begin{aligned}
         S_{\varepsilon}(X_i,\Phi_\eta(X_i)) 
   \leq &\min\left\{e^{\eta D^2 / \lambda^2},\frac{1}{\sqrt{\amin}}\right\} \eta G\left\|\operatorname{grad}_{\operatorname{SHK}} E\left(X_i\right)\right\|_{\operatorname{SHK},X_i} \\
   \leq & \frac{\min\left\{e^{\eta D^2 / \lambda^2},\frac{1}{\sqrt{\amin}}\right\} \eta G^2(N-1) e^{-\beta \Delta}}{1+(N-1) e^{-\beta \Delta}} \\
   \leq & \min\left\{e^{\eta D^2 / \lambda^2},\frac{1}{\sqrt{\amin}}\right\} \eta G^2(N-1) e^{-\beta \Delta}
    \end{aligned}
    $$
    where we used $ \frac{u}{1+u} \leq u$ for $u \geq 0$ in the last inequality.

This completes the proof.
\end{proof}

\section{Discrete retrieval algorithm for empirical measures}\label{Discrete retrieval algorithm for empirical measures}

We now give a fully explicit deterministic retrieval scheme obtained by an explicit Euler discretization of~\eqref{eq:particle-ode}.
The resulting update alternates a Kantorovich (support) step and a spherical Hellinger (weight) step, both driven by the same Sinkhorn computations.

\subsection{Discrete Sinkhorn objects}

Let $\xi=\sum_{m=1}^M a_m\delta_{x_m}$ and $X_i=\sum_{n=1}^{M} b_{i,n}\delta_{y_{i,n}}$.
An entropic coupling is a matrix $P_i\in\mathbb{R}_+^{M\times M}$ with row/column sums $a$ and $b_i$.
The barycentric projection map (as derived in Section \ref{App: entropic potentials}) on the query support is
\[
T^\varepsilon_{\xi\to X_i}(x_m)=\sum_{n=1}^{M}\frac{P_i[m,n]}{a_m}\,y_{i,n},
\qquad
T^\varepsilon_{\xi\to \xi}(x_m)=\sum_{\ell=1}^{M}\frac{P_0[m,\ell]}{a_m}\,x_\ell.
\]
Moreover, the dual potentials returned by Sinkhorn (in the first argument) provide the discrete values $f_{\xi,X_i}(x_m)$ and $f_{\xi,\xi}(x_m)$ needed for~\eqref{eq:z-def}.
For the self-coupling $\operatorname{OT}_\varepsilon(\xi,\xi)$, $\xi$ appears in \emph{both} marginals, so the weight update requires the symmetric combination
\[
f^{\mathrm{sym}}_{\xi,\xi}(x_m)\coloneqq\tfrac12\bigl(f_{\xi,\xi}(x_m)+g_{\xi,\xi}(x_m)\bigr),
\]
where $g_{\xi,\xi}$ is the Sinkhorn potential in the second argument.
For the self-coupling, the correct (gauge-invariant) contribution to the \emph{weight} gradient is the symmetric average $\tfrac12(f_{\xi,\xi}+g_{\xi,\xi})$.

\subsection{Explicit Euler (Kantorovich transport) update}

Using~\eqref{eq:velocity-bary}, a step of size $\eta>0$ updates the support points by
\begin{equation}
\label{eq:x-update}
x_m^{k+1}=x_m^k+\eta\left(\sum_{i=1}^N w_i^k\,T_i^k(x_m^k)-T_0^k(x_m^k)\right),
\end{equation}
where $T_i^k$ denotes the barycentric map computed from the Sinkhorn coupling between $\xi^k$ and $X_i$, and $T_0^k$ is the barycentric map from the self-coupling between $\xi^k$ and itself.

\subsection{Multiplicative (spherical Hellinger) weight update}

A first-order discretization of~\eqref{eq:replicator-ode} can be implemented as the multiplicative update
\begin{equation}
\label{eq:replicator-discrete}
a_m^{k+1}\;\propto\;a_m^k\exp\!\left(-\frac{\eta}{\lambda^2}\,z_m^k\right),
\qquad \sum_{m=1}^M a_m^{k+1}=1,
\end{equation}
where $z_m^k$ is defined by~\eqref{eq:z-def} at $(x^k,a^k)$.
The normalization removes any additive-constant ambiguity in the potentials and guarantees $a_m^{k+1}>0$ whenever $a_m^k>0$.

\subsection{Full deterministic retrieval algorithm}

\textbf{Algorithm 1: Entropic DDAM with spherical Hellinger-Kantorovich retrieval (empirical measures)}

\textbf{Inputs:} stored empirical measures $\{X_i\}_{i=1}^N$; query $\xi^0=\sum_m a_m^0\delta_{x_m^0}$;
parameters $\beta>0$, $\varepsilon>0$, step size $\eta>0$, spherical Hellinger scale $\lambda>0$;
number of iterations $K$ (or a stopping criterion).

\textbf{For $k=0,1,\dots,K-1$ do:}
\begin{enumerate}
    \item \textbf{Sinkhorn couplings and costs.}
    For each $i$, run Sinkhorn between $\xi^k$ and $X_i$ to obtain:
    (a) coupling matrix $P_i^k$,
    (b) source potential values $f_i^k[m]=f_{\xi^k,X_i}(x_m^k)$,
    and (c) the entropic OT cost $\operatorname{OT}_\varepsilon(\xi^k,X_i)$.
    Compute also the self-coupling between $\xi^k$ and itself to obtain $P_0^k$, both potential vectors
    $f_0^k[m]=f_{\xi^k,\xi^k}(x_m^k)$ and $g_0^k[m]=g_{\xi^k,\xi^k}(x_m^k)$ and the entropic self-OT cost $\operatorname{OT}_\varepsilon(\xi^k,\xi^k)$. Finally, for each $i$, compute the OT cost $\operatorname{OT}_\varepsilon(X_i,X_i)$.
    \item \textbf{Sinkhorn divergences and Gibbs weights.}
    For each $i$, compute $S_\varepsilon(\xi^k,X_i)$ using~\eqref{eq:sinkhorn-div}, and set
    \[
    w_i^k=\frac{\exp(-\beta S_\varepsilon(\xi^k,X_i))}{\sum_{j=1}^N\exp(-\beta S_\varepsilon(\xi^k,X_j))}.
    \]
    \item \textbf{Barycentric maps.}
    For each particle $m$ and each pattern $i$, compute
    \[
    T_i^k(x_m^k)=\sum_n \frac{P_i^k[m,n]}{a_m^k}\,y_{i,n},
    \qquad
    T_0^k(x_m^k)=\sum_\ell \frac{P_0^k[m,\ell]}{a_m^k}\,x_\ell^k.
    \]
    \item \textbf{Support update (transport).} Update $x_m^{k+1}$ using~\eqref{eq:x-update}.
    \item \textbf{Weight update (reaction).}
    Compute
    \[
    z_m^k=\sum_{i=1}^N w_i^k\Bigl(f_i^k[m]-\tfrac12\bigl(f_0^k[m]+g_0^k[m]\bigr)\Bigr),
    \]
    then update $(a_m^{k+1})$ by~\eqref{eq:replicator-discrete}.
\end{enumerate}

\textbf{Output:} retrieved empirical measure $\xi^K=\sum_m a_m^K\delta_{x_m^K}$.

All steps above are deterministic given the Sinkhorn solver (which itself is deterministic for fixed initialization and tolerance).
The support update is a pushforward by a deterministic barycentric map, and the weight update is a deterministic multiplicative reweighting on the simplex.
Thus the overall retrieval operator is deterministic.

\section{Auxiliary results}

\subsection{Dual and optimal potentials}

A standard dual form of $\operatorname{OT}_\varepsilon$ can be written (up to equivalent normalizations) in terms of potentials $f,g\in C(\Omega)$.One defines the entropic soft c-transform operator $A_\varepsilon$ via an expression of the form
\[
A_\varepsilon(g,\nu)(x)\coloneqq-\varepsilon\log\int_\Omega \exp\!\Big(\frac{g_{\mu,\nu}(y)-c(x,y)}{\varepsilon}\Big)\,d\nu(y)
\quad \text{(defined up to an additive constant)}.
\]
Then the optimal potentials $(f_{\mu,\nu},g_{\mu,\nu})$ (Schr\"{o}dinger potentials) satisfy the Schr\"{o}dinger system
\[
f_{\mu,\nu}=A_\varepsilon(g_{\mu,\nu},\nu)\quad \mu\text{-a.e.},
\qquad
g_{\mu,\nu}=A_\varepsilon(f_{\mu,\nu},\mu)\quad \nu\text{-a.e.}.
\]

These potentials are unique up to adding a constant to $f_{\mu,\nu}$ and subtracting the same constant from $g_{\mu,\nu}$ (gauge invariance). This does not affect any gradient $\nabla f_{\mu,\nu}(x)$ or $\nabla g_{\mu,\nu}(x)$, which is what we ultimately use.

\subsection{Barycentric projection map}

Let $\pi^\varepsilon_{\mu,\nu}$ be the optimal entropic coupling between $\mu$ and $\nu$, which can be disintegrated into a marginal and conditional distribution as follows :
\[
\pi^\varepsilon_{\mu,\nu}(dx,dy)=\mu(dx)\,\pi^\varepsilon_{\mu,\nu}(dy\mid x).
\]
Define the barycentric projection (conditional mean) map:
\[
T^\varepsilon_{\mu\to\nu}(x)
\coloneqq
\int_\Omega y\,\pi^\varepsilon_{\mu,\nu}(dy\mid x).
\]
This is defined $\mu$-a.e. and takes values in $\mathrm{conv}(\Omega)\subset\mathbb{R}^d$.

\subsection{Computing gradient of Schrödinger potentials explicitly for quadratic costs}\label{sec: Computing gradient of Schrödinger potentials explicitly for quadratic costs}

Recall (from the Schrödinger system) that
\[
f_{\mu,\nu}(x)=
-\varepsilon\log\int_\Omega
\exp\Big(\frac{g_{\mu,\nu}(y)-c(x,y)}{\varepsilon}\Big)\,d\nu(y)
\quad (\text{up to constant}).
\]
Differentiate with respect to $x$. Denote
\[
Z(x)\coloneqq\int_\Omega
\exp\Big(\frac{g_{\mu.\nu}(y)-c(x,y)}{\varepsilon}\Big)\,d\nu(y).
\]
Then $f_{\mu,\nu}(x)=-\varepsilon\log Z(x)$, so
\[
\nabla f_{\mu,\nu}(x)=
-\varepsilon \frac{1}{Z(x)}\nabla Z(x).
\]
Compute $\nabla Z(x)$:
\[
\nabla Z(x)
=\int_\Omega
\exp\Big(\frac{g_{\mu,\nu}(y)-c(x,y)}{\varepsilon}\Big)\cdot
\frac{1}{\varepsilon}\big(-\nabla c(x,y)\big)\,d\nu(y).
\]
Therefore
\[
\begin{aligned}
\nabla f_{\mu,\nu}(x)=&-\varepsilon \frac{1}{Z(x)}
\int
\exp\Big(\frac{g_{\mu,\nu}(y)-c(x,y)}{\varepsilon}\Big)\cdot
\frac{1}{\varepsilon}\big(-\nabla c(x,y)\big)\,d\nu(y)\\
=&\frac{1}{Z(x)}
\int
\exp\Big(\frac{g_{\mu,\nu}(y)-c(x,y)}{\varepsilon}\Big)\,\nabla c(x,y)\,d\nu(y).
\end{aligned}
\]
But the conditional distribution $\pi^\varepsilon(dy\mid x)$ has density proportional to $\exp((g_{\mu,\nu}(y)-c(x,y))/\varepsilon),d\nu(y)$. Hence
\[
\nabla f_{\mu,\nu}(x)=\int \nabla c(x,y)\,\pi^\varepsilon_{\mu,\nu}(dy\mid x).
\]
Now plug $c(x,y)=\frac12\|x-y\|_2^2$, so $\nabla c(x,y)=x-y$. Then
\[
\nabla f_{\mu,\nu}(x)=\int (x-y)\,\pi^\varepsilon_{\mu,\nu}(dy\mid x)=x-\int y\,\pi^\varepsilon_{\mu,\nu}(dy\mid x)=x - T^\varepsilon_{\mu\to\nu}(x).
\]
Thus we have the fundamental identity:
\[
\nabla f_{\mu,\nu}(x)=x-T^\varepsilon_{\mu\to\nu}(x)
\quad\text{for }c(x,y)=\tfrac12\|x-y\|_2^2.
\]
This is precisely why barycentric projections give a transport-map-like representation of entropic OT gradients.

Similarly, we have that $\nabla g_{\mu,\nu}(y) = y - T_{\nu \to \mu}^{^\varepsilon}(y)$. In particular, we have that
\begin{equation}\label{eq: idntity involving self Schordinger potentials}
    \frac{1}{2} \left(\nabla f_{\xi,\xi}(x) + \nabla g_{\xi,\xi}(x)\right) = x - T_{\xi \to \xi}^{^\varepsilon}(x).
\end{equation}

\subsection{Additional lemmas}\label{sec: Additional lemmas}

\begin{lemma}[One step descent of softmin energy]\label{One step descent of energy}
Let Assumptions \ref{Ass: margin sep in Sinkhorn} and \ref{Ass: L-smoothness of energy functional E in SHK geometry} hold. Let $\xi^{+}=\Phi_\eta(\xi)=\operatorname{Ret}_{\xi}(-\eta \operatorname{grad}_{\operatorname{SHK}}E(\xi))$. Then
$$
E\left(\xi^{+}\right) \leq E(\xi)-\eta\left(1-\frac{L \eta}{2}\right)\|\operatorname{grad}_{\operatorname{SHK}}E(\xi)\|_{\operatorname{SHK},\xi}^2 .
$$
In particular, if $0<\eta \leq \frac{1}{L}$, then
$$
E\left(\xi^{+}\right) \leq E(\xi)-\frac{\eta}{2}\|\operatorname{grad}_{\operatorname{SHK}} E(\xi)\|_{\operatorname{SHK},\xi}^2,
$$
so $E$ strictly decreases unless $\operatorname{grad}_{\operatorname{SHK}} E(\xi)=(0,0)$.
\end{lemma}

\begin{proof}
Using Equation \ref{eq: L-smoothness of energy functional E in SHK geometry}, with $(r,v) = - \operatorname{grad}_{\operatorname{SHK}} E(\xi)$, we have that
\[
\begin{aligned}
E\left(\xi^{+}\right) =& E\left(\operatorname{Ret}_{\xi}(- \operatorname{grad}_{\operatorname{SHK}} E(\xi)\right)\\
\leq & E(\xi)-\eta\langle\operatorname{grad}_{\operatorname{SHK}} E(\xi), \operatorname{grad}_{\operatorname{SHK}} E(\xi)\rangle_{\operatorname{SHK},\xi}+\frac{L \eta^2}{2}\| \operatorname{grad}_{\operatorname{SHK}} E(\xi)\|_{\operatorname{SHK},\xi}^2\\
=& E(\xi)-\eta\left(1-\frac{L \eta}{2}\right)\|\operatorname{grad}_{\operatorname{SHK}}E(\xi)\|_{\operatorname{SHK},\xi}^2.
\end{aligned}
\]

If $\eta \leq \frac{1}{L}$, then $1-\frac{L \eta}{2} \geq \frac{1}{2}$, giving the stated bound.
\end{proof}

\begin{lemma}[Energy gap contraction]\label{Energy gap contraction}
Let Assumptions \ref{Ass: margin sep in Sinkhorn}, \ref{Ass: L-smoothness of energy functional E in SHK geometry} and \ref{Ass: PL inequality in local basin} hold and choose $0<\eta \leq \min\left\{\frac{1}{L},\frac{1}{\mu}\right\}$. Then for $i=1,\dots,N$, if $\xi^{(0)} \in B_i(r)$, $\xi^{(k+1)}=\Phi_\eta\left(\xi^{(k)}\right)$ and $\xi^{(k)} \in B_i(r)$ for $k \in \mathbb{N}$, we have

$$
E\left(\xi^{(k+1)}\right)-E_i^{*}(r) \leq\left(1-\mu \eta\right)\left(E\left(\xi^{(k)}\right)-E_i^{*}(r)\right) .
$$

Hence, for any $k \in \mathbb{N}$,

$$
E\left(\xi^{(k)}\right)-E_i^{*}(r)\leq\left(1-\mu \eta\right)^k\left(E\left(\xi^{(0)}\right)-E_i^{*}(r)\right)
$$
and consequently, $\lim_{k \to \infty} E(\xi^{(k)})= E_i^{*}(r)$.

Further, for each $k \in \mathbb{N}$,

$$
\left\|\operatorname{grad}_{\operatorname{SHK}} E\left(\xi^{(k)}\right)\right\|_{\operatorname{SHK}, \xi^{(k)}}^2 \leq \frac{2}{\eta}\left(E\left(\xi^{(k)}\right)-E\left(\xi^{(k+1)}\right)\right) .
$$

Consequently, for each $k \in \mathbb{N}$,
$$
\left\|\operatorname{grad}_{\operatorname{SHK}} E\left(\xi^{(k)}\right)\right\|_{\operatorname{SHK}, \xi^{(k)}} \leq \sqrt{\frac{2}{\eta}\left(E\left(\xi^{(k)}\right)-E_i^{*}(r)\right)} \leq \sqrt{\frac{2 \left(E\left(\xi^{(0)}\right)-E_i^{*}(r)\right)}{\eta}} (1-\eta \mu)^\frac{k}{2} .
$$
    
\end{lemma}

\begin{proof}
    Using Lemma \ref{One step descent of energy}, we have that, provided $\xi^{(k)} \in B_i(r)$,

$$
E\left(\xi^{(k+1)}\right) \leq E\left(\xi^{(k)}\right)-\frac{\eta}{2}\left\|\operatorname{grad}_{\operatorname{SHK}} E\left(\xi^{(k)}\right)\right\|_{\operatorname{SHK}, \xi^{(k)}}^2 .
$$

Subtract $E_i^{*}(r)$ from both sides:

$$
E\left(\xi^{(k+1)}\right)-E_i^{*}(r) \leq E\left(\xi^{(k)}\right)-E_i^{*}(r)-\frac{\eta}{2}\left\|\operatorname{grad}_{\operatorname{SHK}}  E\left(\xi^{(k)}\right)\right\|_{\operatorname{SHK}, \xi^{(k)}}^2 .
$$

Using Equation \ref{eq: PL inequality in local basin}, we have that

$$
\begin{aligned}
\frac{1}{2}\left\|\operatorname{grad}_{\operatorname{SHK}}  E\left(\xi^{(k)}\right)\right\|_{\operatorname{SHK}, \xi^{(k)}}^2 &\geq \mu\left(E\left(\xi^{(k)}\right)-E_i^{*}(r)\right) \\
\Longrightarrow \quad\left\|\operatorname{grad}_{\operatorname{SHK}}  E\left(\xi^{(k)}\right)\right\|_{\operatorname{SHK}, \xi^{(k)}}^2 &\geq 2 \mu\left(E\left(\xi^{(k)}\right)-E_i^{*}(r)\right) .
\end{aligned}
$$

Thus, provided $\xi^{(k)} \in B_i(r)$,

$$
E\left(\xi^{(k+1)}\right)-E_i^{*}(r) \leq E\left(\xi^{(k)}\right)-E_i^{*}(r)-\eta \mu\left(E\left(\xi^{(k)}\right)-E_i^{*}(r)\right)=\left(1-\mu \eta\right)\left(E\left(\xi^{(k)}\right)-E_i^{*}(r)\right) .
$$

Iterating the inequality, we obtain 
$$
0 \leq E\left(\xi^{(k)}\right)-E_i^{*}(r)\leq\left(1-\mu \eta\right)^k\left(E\left(\xi^{(0)}\right)-E_i^{*}(r)\right)
$$
and consequently $\lim_{k \to \infty} E(\xi^{(k)}) = E_i^{*}(r)$.

Using Lemma \ref{One step descent of energy}, since $0 < \eta \leq \frac{1}{L}$, we have that 
$$
E\left(\xi^{(k+1)}\right) \leq E\left(\xi^{(k)}\right)-\frac{\eta}{2}\left\|\operatorname{grad}_{\operatorname{SHK}} E\left(\xi^{(k)}\right)\right\|_{\operatorname{SHK}, \xi^{(k)}}^2
$$
which upon rearrangement gives
$$
\left\|\operatorname{grad}_{\operatorname{SHK}} E\left(\xi^{(k)}\right)\right\|_{\operatorname{SHK}, \xi^{(k)}}^2 \leq \frac{2}{\eta}\left(E\left(\xi^{(k)}\right)-E\left(\xi^{(k+1)}\right)\right) .
$$

Again, since we assume that $\xi^{(k+1)} \in B_i(r)$, we have that $E_i^{*}(r) \leq E\left(\xi^{(k+1)}\right)$. Consequently,
$$
E\left(\xi^{(k)}\right)-E\left(\xi^{(k+1)}\right) \leq E\left(\xi^{(k)}\right)-E_i^{*}(r) .
$$

This completes the proof.

\end{proof}

\begin{lemma}[Control on basin interference and softmin perturbation bounds of energy E inside local basin using separation margin]\label{Control on basin interference and softmin perturbation bounds of energy E inside local basin using separation margin}
    Let Assumption \ref{Ass: margin sep in Sinkhorn} hold. Define the Gibbs weights $w_i(\xi)$ corresponding to any fixed $\xi \in B_i(r)$ as in Equation \ref{softmax weights of energy E}. Then, we have that
    \begin{enumerate}
        \item The weights satisfy
        $$
        w_i(\xi) \geq \frac{1}{1+(N-1) e^{-\beta \Delta}}, \quad \sum_{j \neq i} w_j(\xi) \leq \frac{(N-1) e^{-\beta \Delta}}{1+(N-1) e^{-\beta \Delta}} .
        $$
        \item  The gap between $F_i(\xi)$ and $E(\xi)$ satisfies
        $$
        0 \leq F_i(\xi)-E(\xi) \leq \frac{1}{\beta} \log \left(1+(N-1) e^{-\beta \Delta}\right) \leq \frac{N-1}{\beta} e^{-\beta \Delta}.
        $$

    \end{enumerate}
\end{lemma}

\begin{proof}
    Using Equation \ref{eq: margin sep in Sinkhorn}, we have that, for $\xi \in B_{i}(r)$, 
$$
\sum_{j \neq i} e^{-\beta F_j(\xi)}= e^{-\beta F_i(\xi)} \times \sum_{j \neq i} e^{- \beta\left(F_j(\xi)-F_i(\xi)\right)} \leq e^{-\beta F_i(\xi)} \times\sum_{j \neq i} e^{-\beta \Delta_{r}} \leq e^{-\beta F_i(\xi)} \times (N-1) e^{-\beta \Delta} .
$$

Therefore, we have that
$$
w_i(\xi) = \frac{e^{-\beta F_i(\xi)}}{e^{-\beta F_i(\xi)} + \sum_{j \neq i} e^{- \beta F_j(\xi)}} \geq \frac{1}{1+ (N-1)e^{-\beta \Delta}}
$$

Therefore, we have that $$
\sum_{j \neq i} w_j(\xi) = 1- w_i(\xi) \leq \frac{(N-1)e^{-\beta \Delta}}{1+ (N-1)e^{-\beta \Delta}}.
$$

Now, we have that

$$
E(\xi)=-\frac{1}{\beta} \log \left(e^{-\beta F_i(\xi)}\left[1+\sum_{j \neq i} e^{-\beta\left(F_j(\xi)-F_i(\xi)\right)}\right]\right)=F_i(\xi)-\frac{1}{\beta} \log \left(1+\sum_{j \neq i} e^{-\beta\left(F_j(\xi)-F_i(\xi)\right)}\right) .
$$

Since each $F_j(\xi)-F_i(\xi) \geq \Delta$,

$$
0 \leq \frac{1}{\beta} \log \left(1+\sum_{j \neq i} e^{-\beta\left(F_j-F_i\right)}\right) \leq \frac{1}{\beta} \log \left(1+(N-1) e^{-\beta \Delta}\right) .
$$

Finally, we use $\log (1+u) \leq u$ to get the last bound.
\end{proof}

\begin{lemma}[Expression of SHK gradient oof energy functional E in terms of SHK gradients of Sinkhorn divergence]\label{Expression of SHK gradient oof energy functional E in terms of SHK gradients of Sinkhorn divergence}
    $$
    \operatorname{grad}_{\operatorname{SHK}} E(\xi) = \sum_{j=1}^{N}w_j(\xi) \operatorname{grad}_{\operatorname{SHK}} F_j(\xi)
    $$
\end{lemma}

\begin{proof}
    Proof is obvious.
\end{proof}

\begin{lemma}[Bounding SHK distance in terms of SHK gradient norm along retraction curve]\label{Bounding SHK distance in terms of SHK gradient norm along retraction curve}
    Fix any $(a, x)$ and any tangent increment $(\delta a, \delta x) \in T_{(a, x)} \mathcal{M}_M$. Define the retraction curve

$$
\gamma(t)\coloneqq\operatorname{Ret}_{(a, x)}(t \delta a, t \delta x), \quad t \in[0,1] .
$$

Let $d_{\operatorname{SHK}}$ denote the Riemannian distance induced by the SHK metric on $\mathcal{M}_M = \aspace \times \locspace$. Then

$$
d_{\operatorname{SHK}}\left((a, x), \operatorname{Ret}_{(a, x)}(\delta a, \delta x)\right) \leq \operatorname{Length}(\gamma) \leq \min\left\{e^{\|\frac{\delta a}{a}\|_{\infty}},\frac{1}{\sqrt{\amin}}\right\}\|(\delta a, \delta x)\|_{\operatorname{SHK},(a, x)} .
$$

where $\operatorname{Length}(\gamma) \coloneqq \int_{0}^{1} \|\frac{d}{dt}\gamma(t)\|_{\operatorname{SHK},\gamma(t)} dt$ and $(\frac{\delta a}{a})_m\coloneqq \frac{\delta a_m}{a_m}$.
\end{lemma}

\begin{proof}
    By definition of the Riemannian distance as the infimum of lengths over all curves connecting the points,

    $$
    d_{\operatorname{SHK}}(\gamma(0), \gamma(1)) \leq \operatorname{Length}(\gamma) .
    $$

    So it suffices to bound the length of $\gamma$. We have that
    $\gamma(t)=(a(t), x(t))$ with

    $$
    x(t)=\operatorname{Ret}_x^{\mathrm{pos}}(t\delta x)=x+t \delta x, \quad a(t)=\operatorname{Ret}_a^w(t \delta a) .
    $$

    Then, $x^{\prime}(t) \coloneqq \frac{d}{dt}x(t) = \delta x$. For the weights, introduce
    $$
    s_m\coloneqq\frac{\delta a_m}{a_m}
    $$
    Then, $s \in \mathbb{R}^M$ and $\sum_m a_m s_m=0$ since $\sum_m \delta a_m=0$. By the definition of $\operatorname{Ret}_a^w$, we have that
    $$
    a_m(t)=\frac{a_m e^{t s_m}}{Z(t)}, \quad Z(t)\coloneqq\sum_{\ell=1}^M a_{\ell} e^{t s_{\ell}} .
    $$
    Differentiating, we have that
    $$
    \frac{d}{d t} \log a_m(t)=s_m-\frac{Z^{\prime}(t)}{Z(t)}, \quad \frac{Z^{\prime}(t)}{Z(t)}=\sum_{\ell=1}^M a_{\ell}(t) s_{\ell}=: \bar{s}(t) .
    $$
    Thus
    $$
    a_m^{\prime}(t)\coloneqq \frac{d}{dt} a_m(t)=a_m(t)\left(s_m-\bar{s}(t)\right)
    $$
    The SHK gradient norm at $\gamma(t) = (a(t), x(t))$ is given by
    $$
    \|(u, v)\|_{\operatorname{SHK},(a(t), x(t))}^2=\sum_{m=1}^M\left(\frac{\lambda^2}{a_m(t)} u_m^2+a_m(t)\left\|v_m\right\|^2\right) .
    $$
    Applying it to $(u, v)=\frac{d}{dt}\gamma(t)\coloneqq \frac{d}{dt}\gamma(t) = \left(a^{\prime}(t), x^{\prime}(t)\right)$, we have that
    $$
    \begin{aligned}
    \left\|\frac{d}{dt}\gamma(t)\right\|_{\operatorname{SHK},(a(t), x(t))}^2 =& \sum_{m=1}^M \frac{\lambda^2}{a_m(t)}\left(a_m^{\prime}(t)\right)^2 + \sum_{m=1}^M a_m(t)\left\|x_m^{\prime}(t)\right\|^2\\
    =&\lambda^2 \sum_{m=1}^M a_m(t)\left(s_m-\bar{s}(t)\right)^2 + \sum_{m=1}^M a_m(t)\left\|\delta x_m\right\|^2\\
    =&\lambda^2 \left[\sum_{m=1}^M a_m(t)s_m^2 - 2\bar{s}(t)\sum_{m=1}^{M}a_m(t)s_m + \left(\bar{s}(t)\right)^2\right] + \sum_{m=1}^M a_m(t)\left\|\delta x_m\right\|^2\\
    =&\lambda^2 \left[\sum_{m=1}^M a_m(t)s_m^2- \left(\bar{s}(t)\right)^2\right] + \sum_{m=1}^M a_m(t)\left\|\delta x_m\right\|^2\\
    \leq & \lambda^2 \sum_{m=1}^M a_m(t) s_m^2 + \sum_{m=1}^M a_m(t)\left\|\delta x_m\right\|^2.
    \end{aligned}
    $$
    Let $S\coloneqq\|s\|_{\infty}=\|\delta a / a\|_{\infty}$. Then for all $t \in[0,1]$,
    $$
    e^{-S} \leq e^{t s_m} \leq e^S, \quad e^{-S} \leq Z(t)=\sum_{\ell} a_{\ell} e^{t s_{\ell}} \leq e^S .
    $$
    Hence
    $$
    \frac{a_m e^{-S}}{e^S} \leq a_m(t)=\frac{a_m e^{t s_m}}{Z(t)} \leq \frac{a_m e^S}{e^{-S}},
    $$
    i.e.
    $$
    a_m e^{-2 S} \leq a_m(t) \leq a_m e^{2 S}.
    $$
    Using $a_m(t) \leq a_m e^{2 S}$,
    $$
    \left\|\frac{d}{dt}\gamma(t)\right\|_{\operatorname{SHK},(a(t), x(t))}^2 \leq e^{2 S}\left(\lambda^2 \sum_{m=1}^M a_m s_m^2+\sum_{m=1}^M a_m\left\|\delta x_m\right\|^2\right)=e^{2 S}\|(\delta a, \delta x)\|_{\operatorname{SHK},(a, x)}^2 .
    $$
    Taking square roots gives, for all $t \in [0,1]$,
    $$
    \left\|\frac{d}{dt}\gamma(t)\right\|_{(a(t), x(t))} \leq e^S\|(\delta a, \delta x)\|_{\operatorname{SHK},(a, x)} .
    $$
    Therefore,
    $$
    \operatorname{Length}(\gamma)=\int_0^1\left\|\frac{d}{dt}\gamma(t)\right\|_{\operatorname{SHK},(a(t), x(t))} d t \leq \int_0^1 e^S\|(\delta a, \delta x)\|_{\operatorname{SHK},(a, x)} d t=e^S\|(\delta a, \delta x)\|_{\operatorname{SHK},(a, x)}.
    $$
    Now, note that.
    $$
    \sum_{m=1}^M a_m(t)\left\|\delta x_m\right\|^2 \leq \max _m\left\|\delta x_m\right\|^2 \leq \frac{1}{\amin} \sum_{m=1}^{M} a_m\left\|\delta x_m\right\|^2 
    $$
    and
    $$
    \begin{aligned}
        \lambda^2 \sum_{m=1}^M a_m(t)\left(s_m-\bar{s}(t)\right)^2 \leq & \lambda^2 \sum_{m=1}^M a_m(t)\left(s_m-\bar{s}(0)\right)^2 \\
        \leq & \lambda^2 \max_{m} \left(s_m-\bar{s}(0)\right)^2 \\
        \leq & \frac{\lambda^2}{\amin}\sum_{m=1}^{M} a_m \left(s_m-\bar{s}(0)\right)^2 = \frac{\lambda^2}{\amin}\sum_{m=1}^{M} \frac{(\delta a_m)^{2}}{a_m}.
    \end{aligned}
    $$
    Combining, we have that, 
    $$
        \left\|\frac{d}{dt}\gamma(t)\right\|_{\operatorname{SHK},(a(t), x(t))} \leq \frac{1}{\sqrt{\amin}} \|(\delta a, \delta x)\|_{\operatorname{SHK},(a, x)}
    $$
    and hence,
    $$
        \operatorname{Length}(\gamma) \leq  \frac{1}{\sqrt{\amin}} \|(\delta a, \delta x)\|_{\operatorname{SHK},(a, x)} .
    $$
\end{proof}

\begin{lemma}[Open bounded convex domain implies boundedness of SHK gradient of Sinkhorn divergence]\label{Open bounded convex domain implies boundedness of SHK gradient of Sinkhorn divergence} Assume that the domain $\Omega$ is open, bounded and convex with diameter $D\coloneqq\sup _{x, y \in \Omega}\|x-y\| < \infty$. Then, for $\xi,\nu \in \ProbM$, we have that
\[
\| \operatorname{grad}_{\operatorname{SHK}  }S_{\varepsilon}(\xi,\nu)\|_{\operatorname{SHK},\xi} \leq G.
\]
where $G \coloneqq D \sqrt{1+\frac{D^2}{\lambda^2}}$. In particular, for $i=1,\dots,N$, we consequently obtain 
\[
\|\operatorname{grad}_{\operatorname{SHK}  }F_i(\xi)\|_{\operatorname{SHK},\xi} \leq G.
\]

Further, for any fixed $\xi,\xi^{\prime},\nu \in \ProbM$, the following holds
$$
\left|S_{\varepsilon}(\xi,\nu)-S_{\varepsilon}(\xi^{\prime},\nu)\right| \leq G d_{\operatorname{SHK}}\left(\xi, \xi^{\prime}\right)
$$

where $d_{\operatorname{SHK}}$ is the SHK metric restricted to $\ProbM$. Therefore, $S_{\epsilon}(\cdot,\nu)$ is $G$-Lipschitz continuous in the SHK metric.
    
\end{lemma}

\begin{proof}
    Note that, for any fixed $\nu \in \ProbM$ and $\xi \in \ProbM$, we have that
    \[
    \phi(x) \coloneqq \left(\frac{\delta S_{\varepsilon}(\xi,\nu)}{\delta\xi}\right)(x)
=f_{\xi,\nu}(x)-\tfrac12\bigl(f_{\xi,\xi}(x)+g_{\xi,\xi}(x)\bigr)
    \]
    where $w_i(\xi)$ are defined as in Equation \ref{softmax weights of energy E}.

    Note that $\operatorname{grad}_{\operatorname{SHK}  }S_{\varepsilon}(\xi,\nu)(x) = (r(x),v(x))$ where $v(x) = \nabla \phi(x)$ and $r(x) = \frac{1}{\lambda^2} \left(\phi(x) - \int \phi(y) d\xi(y)\right)$.
    
    Then,
    $$
    \left\|\operatorname{grad}_{\operatorname{SHK}} S_{\varepsilon}(\xi,\nu)\right\|_{\operatorname{SHK},\xi}^2=\int_{\Omega}\|v(x)\|^2 d \xi(x)+\lambda^2 \int_{\Omega} \left[r(x)\right]^2 d \xi(x).
    $$
    Note that,
    $$
        v(x) =  \nabla \left(f_{\xi,\nu}(x) - \frac{1}{2} \left( f_{\xi,\xi}(x) +  g_{\xi,\xi}(x)\right)\right) = - \left[T_{\xi \rightarrow \nu}^\epsilon(x)-T_{\xi \rightarrow \xi}^\epsilon(x) \right].
    $$
    Each barycentric map $T_{\xi \rightarrow \nu}^\epsilon(x)$ lies in $\Omega$, since it is a conditional expectation defined over the domain $\Omega$, and similarly $T_{\xi \rightarrow \xi}^\epsilon(x) \in \Omega$ as well. Hence
    \begin{equation}\label{Bound on norm of gradient of first variation}
    \|v(x)\|= \|\nabla \phi(x) \|=\left\|T_{\xi \rightarrow \nu}^\epsilon(x)-T_{\xi \rightarrow \xi}^\epsilon(x)\right\| \leq D .
    \end{equation}
    Since $\Omega$ is convex and open, every line segment joining two points of $\Omega$ stays in $\Omega$. By the mean value theorem and the bound $\|\nabla \phi(x)\| \leq D$, for any $x, y \in \Omega$ ,
    \begin{equation}\label{Bound on norm of oscillation of first variation}
    |\phi(x)-\phi(y)| \leq \sup _{z \in[x, y]}\|\nabla \phi(z)\|_2\|x-y\|_2 \leq D\|x-y\|_2 \leq D^2 .
    \end{equation}
    Therefore, we have that, 
    $$
    |r(x)| = \frac{1}{\lambda^2}\left|\int \left[\phi(x) -  \phi(y)\right] d\xi(y) \right| \leq \frac{1}{\lambda^2}\int\left| \phi(x) -  \phi(y) \right| d\xi(y)\leq \frac{D^2}{\lambda^2}.
    $$
    Therefore, we have that
    $$
    \left\|\operatorname{grad}_{\operatorname{SHK}} S_{\varepsilon}(\xi,\nu)\right\|_{\operatorname{SHK},\xi}^2 \leq D^2 + \lambda ^2 \times \frac{D^4}{\lambda^4} \leq D^2 +  \frac{D^4}{\lambda^2}.
    $$
    In particular, we can choose $\nu = X_i$. 
    
    For the proof of the Lipschitz continuity of $F_{\nu}(\cdot) \coloneqq S_{\varepsilon}(\cdot,\nu)$, consider $\gamma: [0,1] \to \ProbM$ to be an absolutely continuous path in $\ProbM$ between $\xi$ and $\xi^{\prime}$. Then, by the chain rule,
    $$
    \frac{d}{d t} F_\nu(\gamma(t))=\left\langle\operatorname{grad} F_\nu(\gamma(t)), \frac{d}{dt}\gamma(t)\right\rangle_{\operatorname{SHK}, \gamma(t)}
    $$
    Apply Cauchy-Schwarz and the gradient bound:
    $$
    \left|\frac{d}{d t} F_\nu(\gamma(t))\right| \leq\left\|\operatorname{grad}_{\operatorname{SHK}} F_\nu(\gamma(t))\right\|_{\operatorname{SHK}, \gamma(t)}\left\|\frac{d}{dt}\gamma(t)\right\|_{\operatorname{SHK}, \gamma(t)} \leq G \left\|\frac{d}{dt}\gamma(t)\right\|_{\operatorname{SHK}, \gamma(t)}
    $$
    Integrating from 0 to 1, we have that,
    $$
    \left|F_\nu\left(\xi^{\prime}\right)-F_\nu(\xi)\right| \leq G \int_0^1\left\|\frac{d}{dt}\gamma(t)\right\|_{\mathrm{SHK}, \gamma(t)} d t=G \text { Length }(\gamma)
    $$
    Taking the infimum over all curves $\gamma$ from $\xi$ to $\xi^{\prime}$ gives
    $$
    \left|F_\nu\left(\xi^{\prime}\right)-F_\nu(\xi)\right| \leq G d_{\SHK}\left(\xi, \xi^{\prime}\right)
    $$
\end{proof}

\begin{lemma}[Softmin function sensitivity]\label{softmin function sensitivity}

Define $\softmin_\beta\left(z_1, \ldots, z_N\right)\coloneqq -\frac{1}{\beta} \log \left(\sum_{i=1}^N \exp\left(-\beta z_i\right)\right)$. Then, we have that
$$
\left|\softmin_\beta(z)-\softmin_\beta\left(z^{\prime}\right)\right| \leq\left\|z-z^{\prime}\right\|_{\infty} \quad \forall z, z^{\prime} \in \R^N .
$$
\end{lemma}

\begin{proof}
    Let us define $M_{\operatorname{diff}} \coloneqq \max _{1 \leq i \leq N}\left(z_i-z_i^{\prime}\right)$. Then, we have that $z_i \leq z_i^{\prime} + M_{\operatorname{diff}}$ for $i=1,\dots,N$.
    Note that that $\frac{\partial}{\partial z_i} \softmin_\beta(z) = \frac{\exp(-\beta z_i)}{\sum_{j=1}^{N}\exp(-\beta z_j)} > 0$. Since every partial derivative is positive, the $\softmin$ function is increasing in each variable $z_i$. Therefore, if two vectors $z, w \in \R^N$ satisfy $z_i \leq w_i$ for $i=1,\dots$, then 
    $$
    \softmin_\beta(z) \leq \softmin_\beta(w).
    $$
    Further, for any vector $w \in \R^N$, by direct computation, we have that $\softmin_\beta(w+ C\mathbf{1}) = \softmin_\beta(w_1 + C,\dots, w_N + C) = \softmin_\beta(w) + C$, which show the translation-equivariant nature of $\softmin$.

    Therefore, we have that
    $$
    \softmin_\beta(z)  \leq \softmin_\beta(Z^{\prime}+ M_{\operatorname{diff}}\mathbf{1}) = \softmin_\beta(z) + M_{\operatorname{diff}}.
    $$
    Consequently, we have that,
    $$
    \| \softmin_\beta(z) - \softmin_\beta(z^{\prime})\|_{\infty} \leq M_{\operatorname{diff}} \leq \| z - z^{\prime}\|_{\infty}.
    $$
\end{proof}

\begin{lemma}[Lipschitz continuity of LSE functional with respect to SHK metric]\label{Lipschitz continuity of E with respect to SHK metric}
Assume that the domain $\Omega$ is open, bounded and convex with diameter $D\coloneqq\sup _{x, y \in \Omega}\|x-y\| < \infty$. Then, for stored patterns $X_1,\dots,X_N \in \ProbM$ and $F_i \coloneqq S_{\varepsilon}(\cdot,X_i)$, we have that
    $$
        \left|F_i(\xi)-F_i\left(\xi^{\prime}\right)\right| \leq G d_{\SHK}\left(\xi, \xi^{\prime}\right)
    $$
and
    $$
    \left|E(\xi)-E\left(\xi^{\prime}\right)\right| \leq G d_{\SHK}\left(\xi, \xi^{\prime}\right) \quad \forall \xi, \xi^{\prime} \in P_M(\Omega) .
    $$
    where $G \coloneqq D \sqrt{1+\frac{D^2}{\lambda^2}}$. Therefore,$F_i$ and $E$ are both Lipschitz continuous with respect to the SHK metric with Lipschitz constant $G$, and therefore are continuous as well.
\end{lemma}

\begin{proof}
    With $F_i = S_{\varepsilon}(\cdot,X_i)$, applying Lemma \ref{softmin function sensitivity} with $z_i = F_i(\xi)$ and $z_i^{\prime} = F_i(\xi^{\prime})$, we have that
    $$
    \left|E(\xi)-E\left(\xi^{\prime}\right)\right| \leq \max _{1 \leq i \leq N}\left|F_i(\xi)-F_i\left(\xi^{\prime}\right)\right| .
    $$
    Again, by Lemma \ref{Open bounded convex domain implies boundedness of SHK gradient of Sinkhorn divergence}, we have, for $i=1,\dots,N$,
    $$
        \left|F_i(\xi)-F_i\left(\xi^{\prime}\right)\right| \leq G d_{\SHK}\left(\xi, \xi^{\prime}\right).
    $$
    Taking maximum over $i$ and combining with the previous inequality, we have that
    $$
        \left|E(\xi)-E\left(\xi^{\prime}\right)\right| \leq G d_{\SHK}\left(\xi, \xi^{\prime}\right).
    $$
\end{proof}

\begin{lemma}[Mean separation of patterns implies margin separation of Sinkhorn divergences]\label{mean separation of patterns implies margin separation of Sinkhorn divergences} Let $X_1, \ldots, X_N \in \ProbM$ and define $\mu_i\coloneqq m\left(X_i\right) \in \mathbb{R}^d$. Assume the means are pairwise separated:
$$
\left\|\mu_i-\mu_j\right\| \geq d_{\min } \quad \text { for all } i \neq j .
$$
Fix any $\varepsilon>0$ such that
$$
d_{\min }^2>32 \varepsilon \log M
$$
and define
$$
r\coloneqq\frac{d_{\min }^2}{32}-\varepsilon \log M, \quad \Delta\coloneqq\frac{d_{\min }^2}{4}
$$
Then Assumption \ref{Ass: margin sep in Sinkhorn} holds with the given choice of $r$ and $\Delta$; i.e. for every $i$, for every $\xi \in B_i(r)$, and every $j \neq i$,
$$
F_j(\xi)-F_i(\xi) \geq \Delta.
$$  
\end{lemma}

\begin{proof}
    Fix an index $i$ and let $\xi \in B_i(r)$, so $S_{\varepsilon}\left(\xi, X_i\right) \leq r$. Using Lemma \ref{Sinkhorn divergence lower bound in terms of mean differences} with $(\mu, \nu)=\left(\xi, X_i\right)$, we have that:
$$
S_{\varepsilon}\left(\xi, X_i\right) \geq \frac{1}{2}\left\|m(\xi)-\mu_i\right\|^2-\varepsilon \log M .
$$

Combining with $S_{\varepsilon}\left(\xi, X_i\right) \leq r$, we have that
$$
r \geq \frac{1}{2}\left\|m(\xi)-\mu_i\right\|^2-\varepsilon \log M \Longrightarrow \frac{1}{2}\left\|m(\xi)-\mu_i\right\|^2 \leq r+\varepsilon \log M .
$$
By the definition of $r$, we have that
$r+\varepsilon \log M=\frac{d_{\min }^2}{32}$. Hence
$$
\left\|m(\xi)-\mu_i\right\|^2 \leq 2 \cdot \frac{d_{\min }^2}{32}=\frac{d_{\min }^2}{16} \Longrightarrow\left\|m(\xi)-\mu_i\right\| \leq \frac{d_{\min }}{4} .
$$
Fix $j \neq i$. Then, triangle inequality gives
$$
\left\|m(\xi)-\mu_j\right\| \geq\left\|\mu_i-\mu_j\right\|-\left\|m(\xi)-\mu_i\right\| \geq d_{\min }-\frac{d_{\min }}{4}=\frac{3 d_{\min }}{4} .
$$
Using this along wih Lemma \ref{Sinkhorn divergence lower bound in terms of mean differences}, for $\left(\xi, X_j\right)$, we have that
$$
S_{\varepsilon}\left(\xi, X_j\right) \geq \frac{1}{2}\left\|m(\xi)-\mu_j\right\|^2-\varepsilon \log M \geq \frac{1}{2}\left(\frac{3 d_{\min }}{4}\right)^2-\varepsilon \log M=\frac{9 d_{\min }^2}{32}-\varepsilon \log M.
$$
Since $\xi \in B_i(r), S_{\varepsilon}\left(\xi, X_i\right) \leq r=\frac{d_{\min }^2}{32}-\varepsilon \log M$. Hence , we have that
$$
S_{\varepsilon}\left(\xi, X_j\right)-S_{\varepsilon}\left(\xi, X_i\right) \geq\left(\frac{9 d_{\min }^2}{32}-\varepsilon \log M\right)-\left(\frac{d_{\min }^2}{32}-\varepsilon \log M\right)=\frac{8 d_{\min }^2}{32}=\frac{d_{\min }^2}{4}=\Delta .
$$
 Thus, Assumption \ref{Ass: margin sep in Sinkhorn} holds under the given conditions.   
\end{proof}

\begin{lemma}[Entropic OT lower bound in terms of mean differences]\label{Entropic OT lower bound in terms of mean differences}
    For any $\mu, \nu \in \mathcal{P}(\Omega)$ and any $\varepsilon>0$,
    $$
    \operatorname{OT}_{\varepsilon}(\mu, \nu) \geq \frac{1}{2}\|m(\mu)-m(\nu)\|^2
    $$
    where $m(\mu)\coloneqq\int_{\Omega} x d \mu(x) \in \operatorname{conv}(\Omega) \subset \mathbb{R}^d$ is the mean of $\mu$.
\end{lemma}

\begin{proof}
    Given any coupling  $\pi \in \Pi(\mu, \nu)$ and random pair $(X,Y) \sim \pi$, define $Z \coloneqq X - Y$. Then, we have that 
    $$
    \E_\pi[Z]=\E_\pi[X]-\E_\pi[Y]=m(\mu)-m(\nu).
    $$
    By Jensen's inequality, we have that
    $$
    \E_\pi\left[\|Z\|^2\right] \geq\left\|\E_\pi[Z]\right\|^2=\|m(\mu)-m(\nu)\|^2 .
    $$

    Consequently, using the fact $\KLf{\pi}{\mu \otimes \nu} \geq 0$, we have that
    $$
    \begin{aligned}
    \int_{\Omega \times \Omega} c(x, y) d \pi(x, y) +\varepsilon \KLf{\pi}{\mu \otimes \nu} \geq & \int_{\Omega \times \Omega} c(x, y) d \pi(x, y)\\
    =&\frac{1}{2} \mathbb{E}_\pi\left[\|X-Y\|^2\right] \\
    \geq & \frac{1}{2}\|m(\mu)-m(\nu)\|^2 .
    \end{aligned}
    $$

    Since this is true for any $\pi \in \Pi(\mu, \nu)$, taking the infimum over $\pi \in \Pi(\mu, \nu)$, we have that
    $$
    \operatorname{OT}_{\varepsilon}(\mu, \nu) \geq \frac{1}{2}\|m(\mu)-m(\nu)\|^2 .
    $$
\end{proof}

\begin{lemma}[Self Entropic OT distance upper bound]\label{Self Entropic OT distance upper bound}
    If $\mu=\sum_{m=1}^M a_m \delta_{x_m} \in \ProbM$ with $a_m>0$, then
    $$
    \operatorname{OT}_{\varepsilon}(\mu, \mu) \leq \varepsilon \log M .
    $$
\end{lemma}

\begin{proof}
    Let $\pi_{\mathrm{diag}} \coloneqq \sum_{m=1}^M a_m \delta_{\left(x_m, x_m\right)} \in \Pi(\mu, \mu) $ denote the identity/diagonal coupling and its corresponding transport cost is given by
    $$
    \int c(x,y) d \pi_{\mathrm{diag}}(x,y)=\sum_{m=1}^M a_m \times \frac{1}{2}\left\|x_m-x_m\right\|^2=0 .
    $$

    On the discrete support $\left\{\left(x_m, x_{\ell}\right)\right\}_{m, \ell}$, we have that
    $$
    (\mu \otimes \mu)\left(x_m, x_{\ell}\right)=a_m a_{\ell} \quad \textrm{and} \quad \pi_{\mathrm{diag}}\left(x_m, x_{\ell}\right)= \begin{cases}a_m, & m=\ell, \\ 0, & m \neq \ell .\end{cases}
    $$
    Consequently, we have that,
    $$
    \KLf{\pi_{\mathrm{diag}}}{\mu \otimes \mu} = \sum_{m=1}^M a_m \log \frac{a_m}{a_m^2}=\sum_{m=1}^M a_m \log \frac{1}{a_m}=H(a)
    $$
    where $H(a) \coloneqq - \sum_{m=1}^M a_m \log a_m$ is the Shannon entropy of the coupling $\pi_{\mathrm{diag}}$ corresponding to the probability vector $a=(a_1,\dots,a_m)$ or equivalently, that of the discrete distribution $\sum_{m=1}^M a_m \delta_{x_m}$. Using the strict concavity of $x \mapsto \log x$ and Jensen's inequality \cite[Theorem 1.4(b)]{polyanskiy2025information}, we have that
    $$
        H(a) = \E_{X \sim \mu} \log \left[\frac{1}{\mu(X)}\right] \leq \log \E \left[\frac{1}{\mu(X)}\right]= \log M.
    $$
    Therefore, we have that,
    $$
    \operatorname{OT}_{\varepsilon}(\mu, \mu) \leq \int c(x,y) d \pi_{\mathrm{diag}}(x,y)+\varepsilon \KLf{\pi_{\mathrm{diag}}}{\mu \otimes \mu}=0+\varepsilon H(a) \leq \varepsilon \log M .
    $$
\end{proof}

\begin{lemma}[Sinkhorn divergence lower bound in terms of mean differences]\label{Sinkhorn divergence lower bound in terms of mean differences}
    For any $\mu, \nu \in \ProbM$,

$$
S_{\varepsilon}(\mu, \nu) \geq \frac{1}{2}\|m(\mu)-m(\nu)\|^2-\varepsilon \log M .
$$
\end{lemma}

\begin{proof}
    Using Lemma \ref{Entropic OT lower bound in terms of mean differences} for the first term and Lemma \ref{Self Entropic OT distance upper bound} for the last two terms, we have, from the definition of $S_{\varepsilon}(\mu, \nu)$ :
    $$
    \begin{aligned}
    &S_{\varepsilon}(\mu, \nu) \\
    =& \operatorname{OT}_{\varepsilon}(\mu, \nu)-\frac{1}{2} \operatorname{OT}_{\varepsilon}(\mu, \mu)-\frac{1}{2} \operatorname{OT}_{\varepsilon}(\nu, \nu)\\
    \geq & \frac{1}{2}\|m(\mu)-m(\nu)\|^2-\frac{1}{2}(\varepsilon \log M)-\frac{1}{2}(\varepsilon \log M)\\
    =& \frac{1}{2}\|m(\mu)-m(\nu)\|^2-\varepsilon \log M .
    \end{aligned}
    $$
\end{proof}

\begin{lemma}[pairwise separation of random sign vectors]\label{pairwise separation of random sign vectors}
    Let $s_i, s_j \in\{ \pm 1\}^d$ be componentwise independent Rademacher random variables and define
$$
\mu_i=c+\frac{R_0}{\sqrt{d}} s_i, \quad \mu_j=c+\frac{R_0}{\sqrt{d}} s_j .
$$
Fix $\gamma \in(0,1)$. Then
$$
\mathbb{P}\left(\left\|\mu_i-\mu_j\right\|^2<2(1-\gamma) R_0^2\right) \leq \exp \left(-\frac{\gamma^2}{2} d\right) .
$$
\end{lemma}

\begin{proof}
    Define the random variable $H \coloneqq \sum_{k=1}^d \mathbf{1}\left\{s_{ik} \neq s_{jk}\right\}$. For any $k =1,\dots,d$, since $s_{ik}$ and $s_{jk}$ are independent Rademacher random variables, $\mathbf{1}\left\{s_{ik} \neq s_{jk}\right\}$ is a Bernoulli random variable with probability parameter $p=\frac{1}{2}$. Using the independence across $k$, we have that $H$ is the sum of $d$ i.i.d Bernoulli random variables with common probability parameter $p=\frac{1}{2}$. Consequently, $H \sim \operatorname{Binomial}(d,\frac{1}{2})$ and $\E(H) = \frac{d}{2}$.

    Note that $|s_{ik} - s_{jk}| = 0$ if they are equal and $2$ is different. Further, $ \mu_i-\mu_j=\frac{R_0}{\sqrt{d}}\left(s_i-s_j\right)$. Therefore
    $$
    \left\|\mu_i-\mu_j\right\|^2=\frac{R_0^2}{d} \sum_{k=1}^d\left(s_{i, k}-s_{j, k}\right)^2=\frac{4 R_0^2}{d} H
    $$
    Thus, the event $\left\|\mu_i-\mu_j\right\|^2<2(1-\gamma) R_0^2$ is equivalent to
    $$
    \frac{4 R_0^2}{d} H<2(1-\gamma) R_0^2 \quad \Longleftrightarrow \quad H<\frac{1-\gamma}{2} d = \frac{d}{2} - \frac{\gamma d}{2}.
    $$
    Applying Hoeffding's inequality to $H$, we have that, for any $t>0$,
    $$
    \begin{aligned}
    & \mathbb{P}(H-\E H \leq-t) \leq \exp \left(-\frac{2 t^2}{d}\right) \\
    \iff& \mathbb{P}(H-\frac{d}{2} \leq-t) \leq \exp \left(-\frac{2 t^2}{d}\right).
    \end{aligned}
    $$
    Choosing $t=\frac{\gamma d}{2}$, we have that
    $$
    \begin{aligned}
    & \mathbb{P}(H-\E H \leq-\frac{\gamma d}{2}) \leq \exp \left(-\frac{\gamma^2}{2} d\right) \\
    \iff& \mathbb{P}(H \leq\frac{d}{2}-\frac{\gamma d}{2}) \leq \exp \left(-\frac{\gamma^2}{2} d\right)\\
    \iff& \mathbb{P}(H \leq\frac{1-\gamma}{2} d) \leq \exp \left(-\frac{\gamma^2}{2} d\right)\\
    \iff& \mathbb{P}\left(\left\|\mu_i-\mu_j\right\|^2<2(1-\gamma) R_0^2\right) \leq \exp \left(-\frac{\gamma^2}{2} d\right) .
    \end{aligned}
    $$
\end{proof}

\begin{lemma}[uniform separation for all pairs of means of stored patterns]\label{uniform separation for all pairs of means of stored patterns}
    Let $\mu_1, \ldots, \mu_N$ be defined as in Lemma \ref{pairwise separation of random sign vectors} using the i.i.d. componentwise Rademacher random variables $s_1,\dots,s_N \in\{ \pm 1\}^d$. Fix $p \in(0,1)$ and $\gamma \in(0,1)$. If
$$
N\coloneqq\left\lfloor\sqrt{2 p} \exp \left(\frac{\gamma^2}{4} d\right)\right\rfloor .
$$
then with probability at least $1-p$,
$$
\left\|\mu_i-\mu_j\right\| \geq d_{\min }\coloneqq\sqrt{2(1-\gamma)} R_0 \quad \text { for all } i \neq j .
$$
\end{lemma}

\begin{proof}
    Let $A_{i j}$ be the bad event $\left\{\left\|\mu_i-\mu_j\right\|^2<2(1-\gamma) R_0^2\right\}$. By Lemma \ref{pairwise separation of random sign vectors},
    $$
    \mathbb{P}\left(A_{i j}\right) \leq \exp \left(-\frac{\gamma^2}{2} d\right) \quad \text { for each } i \neq j .
    $$
    By the union bound over $\binom{N}{2} \leq \frac{N^2}{2}$ pairs,
    $$
    \mathbb{P}\left(\exists \,\ i<j: A_{i j}\right) \leq \frac{N^2}{2} \exp \left(-\frac{\gamma^2}{2} d\right) .
    $$
    Under the stated condition on $N$, the RHS is $\leq p$. Therefore, with probability at least $1-p$, no bad event occurs, i.e. all pairs satisfy
    $$
    \left\|\mu_i-\mu_j\right\|^2 \geq 2(1-\gamma) R_0^2 \Longleftrightarrow\left\|\mu_i-\mu_j\right\| \geq \sqrt{2(1-\gamma)} R_0=d_{\min } .
    $$
\end{proof}

\begin{lemma}[Linear independence of Dirac distribution and its distributional derivative]\label{Linear independence of Dirac distribution and its distributional derivative}
Let $(x_1, \ldots, x_M) \in \Omega \subset \R^d$ be pairwise distinct location parameters. Suppose 
\begin{equation*}
\sum_{i=1}^M c_i \delta_{x_i}+\sum_{i=1}^M v_i \cdot \nabla \delta_{x_i}=0 \quad \text { in } \left(C^{\infty}(\R^d)\right)^{\prime},
\end{equation*}
with $c_i \in \R$ and $v_i \in \R^d$. Then $c_i=0$ and $v_i=0$ for all $i =1,\dots, M$.
\end{lemma}

\begin{proof}
    Fix $k \in\{1, \ldots, M\}$. Given any $x \in \Omega$ and $r\geq 0$, define the Euclidean ball $B(x,r) \coloneqq \left\{y \in \Omega :\|y-x\|\leq r\right\}$. Since the $x_i$'s are distinct, there exists $r_k>0$ such that
    $B\left(x_k, r_k\right) \cap\left\{x_j: j \neq k\right\}=\emptyset$. Let us choose $\psi \in C_c^{\infty}\left(\R^d\right) \subset C^{\infty}(\R^d)$ supported in $B\left(x_k, r_k\right)$ and equal to 1 in a neighborhood of $x_k$. Then, we have that, $\left\langle\delta_{x_i}, \psi\right\rangle= \begin{cases}\psi(x_i) =1, & i=k \\ \psi(x_k) = 0, & i \neq k\end{cases}$ and $\left\langle\nabla \delta_{x_i}, \psi\right\rangle=-\nabla \psi\left(x_i\right)=0$ for $i =1,\dots,M$, since $\psi$ is constant near $x_k$ and vanishes near $x_i$ for $i \neq k$. Consequently, we have that 
    \begin{equation*}
    \begin{aligned}
        &\sum_{i=1}^M \langle c_i \delta_{x_i}, \psi \rangle+\sum_{i=1}^M \langle v_i \cdot \nabla \delta_{x_i}, \psi \rangle=\langle 0, \psi \rangle \\
        \iff & \sum_{i=1}^{M} c_i \psi(x_i) - \sum_{i=1}^{M} v_i \cdot\nabla \psi(x_i) = 0 \\
        \iff &c_k = 0.
    \end{aligned}
    \end{equation*}

    Now, let us fix any vector $u \in x \subset \R^d$. Then, let us choose $\psi_u \in C_c^{\infty}\left(\R^d\right) \subset C^{\infty}(\R^d)$ supported in $B(x_k, r_k)$ with $\psi_u(x_k)=0$ and $ \nabla \psi_u(x_k)=u$. Then, we have that $ \left\langle\delta_{x_i}, \psi_u\right\rangle=0 \quad \text{ for } i=1,\dots,M$ and $\left\langle\nabla \delta_{x_i}, \psi_u\right\rangle=-\nabla \psi_u(x_i)= \begin{cases} -\nabla \psi_u(x_k)=-u, & i=k, \\ -\nabla \psi_u(x_i) = 0, & i \neq k \end{cases}$. Consequently, we have that 
    \begin{equation*}
    \begin{aligned}
        &\sum_{i=1}^M \langle c_i \delta_{x_i}, \psi_u \rangle+\sum_{i=1}^M \langle v_i \cdot \nabla \delta_{x_i}, \psi_u \rangle=\langle 0, \psi_u \rangle \\
        \iff & \sum_{i=1}^{M} c_i \psi_u(x_i) - \sum_{i=1}^{M} v_i \cdot \nabla \psi_u(x_i) = 0 \\
        \iff & - v_k \cdot u = 0.
    \end{aligned}
    \end{equation*}
    Since this is true for any $u \in \Omega \in \R^d$, we must have that $v_k =0$. Finally, since $k$ was arbitrary, we must have that $c_i=0$ and $v_i=0$ for all $i =1,\dots, M$.
\end{proof}

\section{Numerical Experiments}\label{Numerical Experiments}

In this section, we demonstrate the empirical performance of our proposed retrieval algorithm, referred to as SinkhornSHK Algo, for finitely supported discrete measures and compare it to a baseline Euclidean geometry based classical Hopfield-type algorithm, which we refer to as Euclidean Algo. The Euclidean algo vectorizes $(a,x) \in  \mathcal{M}_M$ into \\ $\xi_{\operatorname{vec}} = \left[x_{11},\dots,x_{1d},\dots,x_{M1},\dots,x_{Md},\log a_1, \dots, \log a_m\right] \in \mathbb{R}^{(d+1)M}$ and applied the classical Hopfield fixed point algorithm for vector inputs based on Euclidean $\ell_2$ inner product similarity using Equation 3 of \cite{ramsauer2020hopfield} with the same choice of $\beta$ as for our proposed SinkhornSHK Algo.

We consider a toy experiment where the stored patterns $X_1,\dots,X_N$ are uniformly weighted and the support points are sampled from Gaussian distributions. We choose $N=5$ and data dimension $d=2$. 

In Experiment 1, we choose the means of the Gaussian distributions to be $(-4.0, -1.0)$,$(-2.0, 2.2)$,\\$(1.0, -6.0)$,$(4.0, -4.2)$ and $(4.2, -0.8)$, while the covariance matrices were chosen to be $
\begin{bmatrix}
0.60 & 0.20 \\
0.20 & 0.90
\end{bmatrix}$,\\$\begin{bmatrix}
0.80 & -0.15 \\
-0.15 & 0.55
\end{bmatrix},\begin{bmatrix}
0.65 & 0.00 \\
0.00 & 0.65
\end{bmatrix},\begin{bmatrix}
0.55 & 0.10 \\
0.10 & 1.00
\end{bmatrix}$ and $\begin{bmatrix}
0.95 & 0.00 \\
0.00 & 0.50
\end{bmatrix}$. $M=30$ support points were sampled in i.i.d manner from the 5 Gaussian distributions determined by each pair of mean and covariance parameters and the resulting uniformly weighted discrete distributions were set as the patterns to be stored.

In Experiment 2, we choose the means of the 5 Gaussian distributions to be all equal to $(0,0)$ and the covariance matrices were randomly sampled using random orthogonal matrices coupled with uniformly sampled eigenvalues between 0.15 and 1.75. We sample $M=25$ support points in i.i.d manner from each of these Gaussian distributions and the resulting uniformly weighted discrete distributions were set as the patterns to be stored.

In both the experiments, we first fix a pattern that we want to retrieve, then perturb the support points individually using i.i.d Gaussian noise (sd 0.5 in Experiment 1 and sd 0.2 in Experiment 2) to generate a query distribution that serves as the initial iterate $\xi^{(0)}$ for both algorithms. We chose $\beta=50$, $\varepsilon=0.05$ and step-size $\eta = 1.3$ for the SinkhornSHK Algo, and the same $\beta$ for Euclidean Algo. We do not use the spherical Hellinger update step in these simple experiments since the all discrete measures involved are uniformly weighted. We use a maximum iteration threshold $k \leq 200$ for both algorithms, and the Sinkhorn algorithm for computing entropic OT transport plans were capped at 120 iterations.

In Experiment 1, we see that both algorithms are able to retrieve the correct discrete distributions when considering the support points that are returned by either algorithm. However, Experiment 2 clearly shows the superiority of SinkhornSHK Algo over Euclidean Algo, since the Sinkhorn Algo is able to converge to the correct pattern even when a noisy query is given. We believe that the ability of Sinkhorn ALgo to leverage the distributional perspective gives it the advantage over Euclidean Algo, since the latter relies on Euclidean inner products and is expected to fail in cases where Euclidean separation between support points is small, but separation in distributional metrics is still feasible.
\begin{figure}[!htbp]
  \centering
    \includegraphics[width=1.2\linewidth]{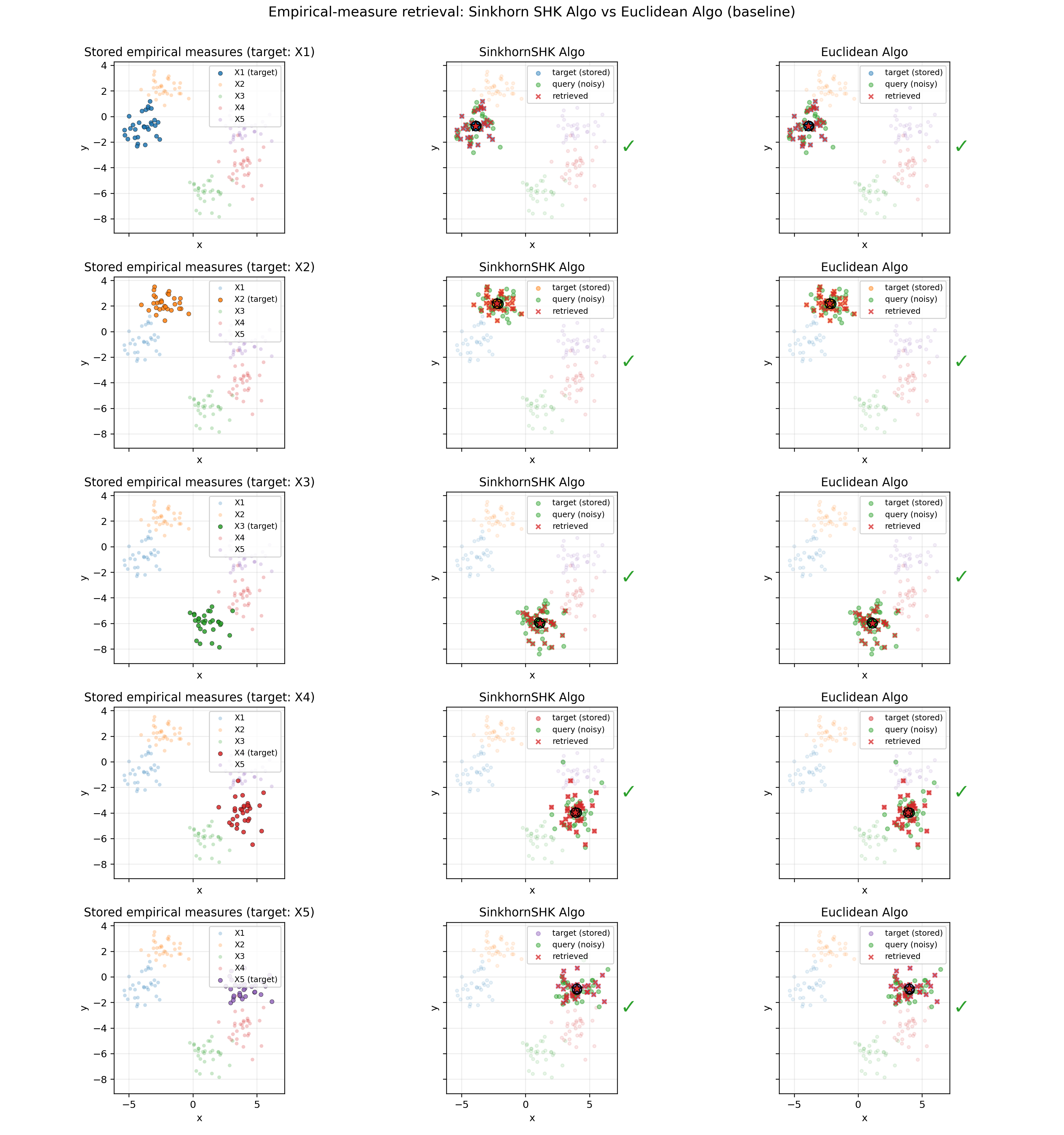}
   \caption{Experiment 1: Sinkhorn Algo and Euclidean Algo are both able to retrieve correct patterns from noisy queries}
    \label{Experiment 1: Sinkhorn Algo and Euclidean Algo are both able to retrieve correct patterns from noisy queries}
\end{figure}

\begin{figure}[!htbp]
  \centering
    \includegraphics[width=1.2\linewidth]{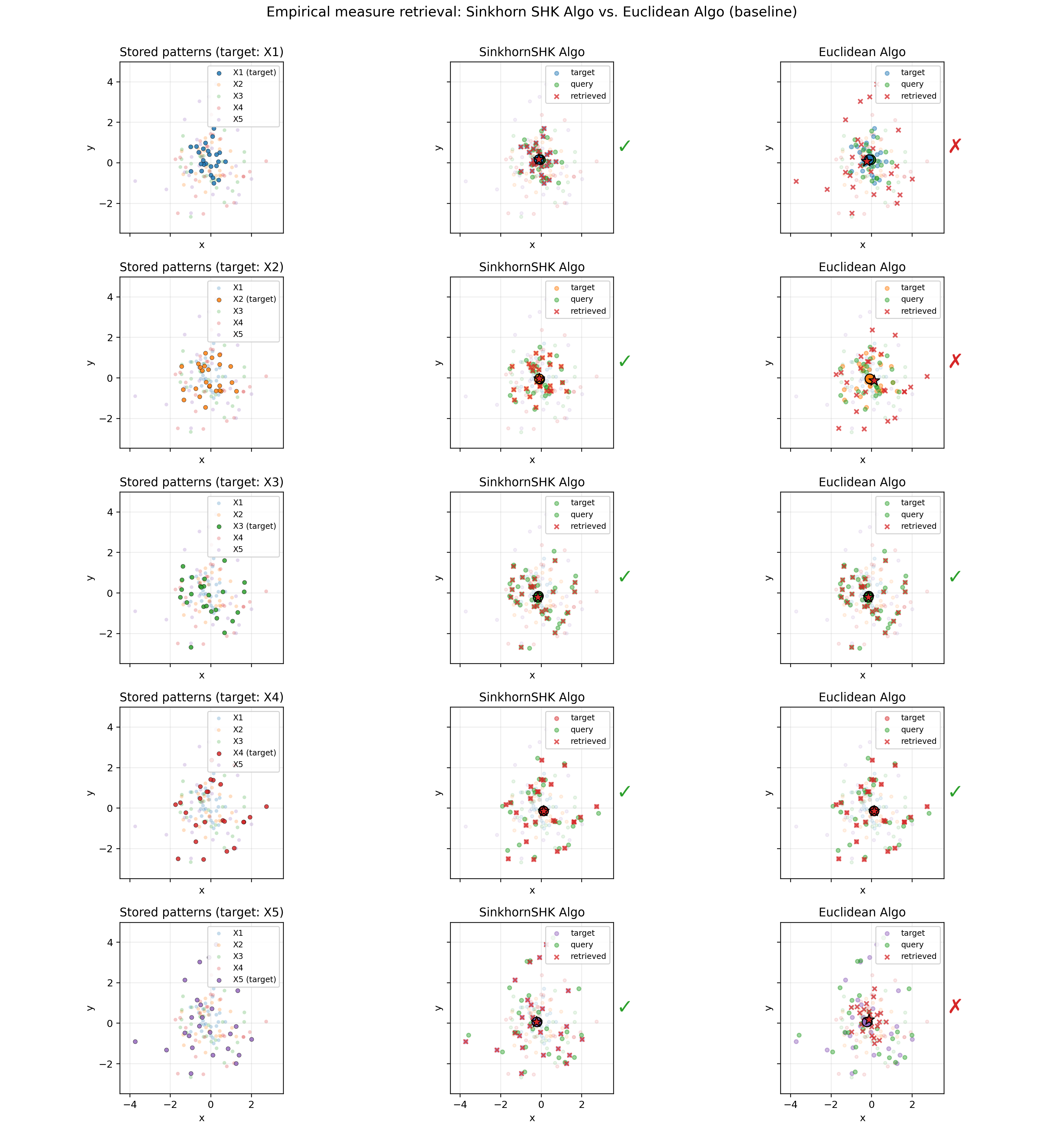}
    \caption{Experiment 2: Sinkhorn Algo succeeds in retrieving correct patterns from noisy queries in all instances, but Euclidean Algo fails in 3 cases.}
\label{Experiment 2: Sinkhorn Algo succeeds in retrieving correct patterns from noisy queries in all instances, but Euclidean Algo fails in 3 cases.}
\end{figure}
%%%%%%%%%%%%%%%%%%%%%%%%%%%%%%%%%%%%%%%%%%%%%%%%%%%%%%%%%%%%

\end{document}